\documentclass{article}
\usepackage{graphicx} % Required for inserting images
\usepackage{amsmath}
\usepackage{amssymb}
\usepackage{amsthm}
\usepackage{wrapfig}
\usepackage{tcolorbox}
\usepackage{xcolor}
\usepackage{amsfonts}
\usepackage{caption}
\usepackage{subcaption}
\usepackage{lipsum}

\usepackage{natbib}
\setcitestyle{authoryear,round,citesep={;},aysep={,},yysep={;}}

\usepackage{algorithm, algorithmic}
\usepackage{enumerate}
\usepackage{framed}
\usepackage{verbatim}
\usepackage{color}
\usepackage{microtype}

\usepackage{thm-restate}
\newtheorem{theorem}{Theorem}

\newtheorem{lemma}[theorem]{Lemma}

\theoremstyle{definition} % CHANGE IF NEEDED
\newtheorem{definition}{Definition}

\usepackage[utf8]{inputenc} % allow utf-8 input
\usepackage[T1]{fontenc}    % use 8-bit T1 fonts
\usepackage{hyperref}       % hyperlinks
\usepackage{url}            % simple URL typesetting
\usepackage{booktabs}       % professional-quality tables
\usepackage{amsfonts}       % blackboard math symbols
\usepackage{nicefrac}       % compact symbols for 1/2, etc.
\usepackage{microtype}      % microtypography
\usepackage{xcolor}         % colors

\newenvironment{squishitemize}
{\begin{list}{\textbullet}{%
    \setlength{\itemsep}{0pt}%
    \setlength{\parsep}{0pt}%
    \setlength{\topsep}{0pt}%
    \setlength{\parskip}{0pt} %
    \setlength{\labelwidth}{.5in}%
    \setlength{\labelsep}{0.05in} %
    \setlength{\leftmargin}{.15in} %
    }}
  {\end{list}}

  \newenvironment{squishenumerate}
  {\begin{list}{\arabic{enumi}.}{%
    \usecounter{enumi}%
    \setlength{\itemsep}{0pt}%
    \setlength{\parsep}{0pt}%
    \setlength{\topsep}{0pt}%
    \setlength{\parskip}{0pt}%
    \setlength{\labelwidth}{.5in}%
    \setlength{\labelsep}{0.05in}%
    \setlength{\leftmargin}{.2in}}}
  {\end{list}}

\newcommand{\calA}{\mathcal{A}}
\newcommand{\calB}{\mathcal{B}}

\newcommand{\calQ}{\mathcal{Q}}
\newcommand{\calR}{\mathcal{R}}

\def\reals{{\mathbb R}}
\def\naturals{{\mathbb N}}

\newcommand{\norm}[1]{\left\| #1 \right\|}

\newcommand{\eps}{\epsilon}

\newcommand{\tvd}{\mathrm{d_{TV}}}
\newcommand{\univ}{\mathcal{Z}}
\newcommand{\dist}{\mathbf{P}}
\newcommand{\samp}{z}
\newcommand{\sample}{\mathbf{\samp}}

\newcommand{\query}{q}

\newcommand{\queries}{k}
\newcommand{\queryset}{Q}

\newcommand{\accgame}{\mathsf{Acc}}

\newcommand{\mech}{\calA}
\newcommand{\adv}{\calQ}

\newcommand{\mem}{\texttt{mem}}
\newcommand{\curv}{\texttt{curv}}

\newcommand{\eqdef}{\mathrel{\mathop:}=}

\newcommand{\pr}[2]{\underset{#1}{\mathbb{P}}\left[ #2 \right]}
\newcommand{\ex}[2]{\underset{#1}{\mathbb{E}}\left[ #2 \right]}

\newcommand{\getsr}{\gets_{\mbox{\tiny R}}}

\newcommand{\from}{:}
\newcommand{\eqand}{\qquad \textrm{and} \qquad}

\newcommand\extrafootertext[1]{%
    \bgroup
    \renewcommand\thefootnote{\fnsymbol{footnote}}%
    \renewcommand\thempfootnote{\fnsymbol{mpfootnote}}%
    \footnotetext[0]{#1}%
    \egroup
}

\usepackage{dsfont}
\newcommand{\ind}{\mathds{1}}
\usepackage{arxiv, times}

\title{Efficiently Attacking Memorization Scores}

\author{Tue Do, Varun Chandrasekaran, Daniel Alabi}
\begin{document}

\maketitle

\begin{abstract}
Influence estimation tools—such as memorization scores—are widely used to understand model behavior, attribute training data, and inform dataset curation. However, recent applications in data valuation and responsible machine learning raise the question: can these scores themselves be adversarially manipulated? In this work, we present a systematic study of the feasibility of attacking memorization-based influence estimators. We characterize attacks for producing highly memorized samples as highly sensitive queries in the regime where a trained algorithm is accurate. Our attack (calculating the pseudoinverse of the input) is practical, requiring only black-box access to model outputs and incur modest computational overhead. We empirically validate our attack across a wide suite of image classification tasks, showing that even state-of-the-art proxies are vulnerable to targeted score manipulations. In addition, we provide a theoretical analysis of the stability of memorization scores under adversarial perturbations, revealing conditions under which influence estimates are inherently fragile. Our findings highlight critical vulnerabilities in influence-based attribution and suggest the need for robust defenses. All code can be found at \url{https://github.com/tuedo2/MemAttack}
\end{abstract}

\section{Introduction}
\label{sec:intro}

Online data market platforms, such as AWS Data Exchange~\citep{AWSD}, Dawex~\citep{dawex}, Xignite~\citep{xignite}, WorldQuant~\citep{worldquant}, are spaces where data is bought and sold. 
%Just as a physical market allows buyers and sellers to exchange goods, a data market facilitates the exchange of data. 
Concretely, there are three major entities in a data market: {platforms}, {buyers}, and {sellers/providers}~\citep{dm:survey:kennedy2022revisiting, AGMSW25}. The data market platform performs {\em data valuation} based on the acquired data from data sellers~\citep{dm:survey:reallife:azcoitia2022survey, pricing:mehta2021sell, dm:agarwal2019marketplace, decentralized_trust:fan2020decentralized, valuation:wang2023data}. 
Sellers in data markets offer datasets (collections of information) that are valuable for businesses, researchers, or governments. Buyers look for data that can help them make better decisions, run more effective marketing campaigns, or improve products and services. 
{\em Data valuation} is released to buyers who use that information to buy data~\citep{negotiation:jung2019privacy, negotiation:ray2020bargaining}. 

Influence functions such as Shapley values are commonly used to price data (proportional to its valuation) in a data
market~\citep{valuation:yan2021if, valuation:fl:song2019profit, valuation:fl:wang2020principled}. A series of recent papers~\citep{feldman2020neural,feldman2020does,brown2021memorization} propose the {\em label memorization score} for supervised classification (in machine learning settings): in large datasets, a small subset of highly influential (memorized) training examples disproportionately affects the model’s predictions and generalization capabilities, while the majority of examples have little to no impact. %In other words, a small fraction of the dataset (the ``head'') significantly influences the model’s predictions and generalization ability, while most data points (the ``tail'') contribute minimally.
Clearly, this concept is relevant in data valuation, where the goal is to identify which training data points contribute most to a model’s decision-making (i.e., samples with high memorization scores are more valuable). %Intuitively, data entries with high memorization scores are underrepresented by the overall dataset and thus their individual bearing on downstream performance is higher overall than generic representative. 
While the original proposal~\citep{feldman2020neural} is computationally expensive, various studies~\citep{garg2023memorization,ravikumar2024unveiling,jiang2020characterizing,zhao2024scalability} propose efficient proxies.

Given the difficulty of obtaining high-quality data in data market platforms (and consequently their high value),
there is a clear economic incentive for some sellers to manipulate the valuation scores~\citep{AGMSW25}. 

\begin{tcolorbox}
We aim to study how malicious data sellers 
can alter 
influence valuations in data markets. In particular, we look at the robustness of \textit{memorization-based} approaches, by proposing multiple attacks based on distribution shift, stability notions, and decision boundary proximity.
\end{tcolorbox}

% In particular, as in the classic robustness setting, a fraction of sellers can collude so that they produce new samples that would have very high memorization score, thus resulting in higher valuation compared to all other (honest) data providers. 
\citet{basu2021influence} previously study the robustness of influence functions, but do not provide theoretical guarantees and restrict their empirical analyses to neural networks. 
% \begin{wrapfigure}{r}{0.6\textwidth} % side: %\begin{figure}
%     \centering
%     \includegraphics[width=0.95\linewidth]{assets/memattackdiagram.png}
%     \caption{An overview of how data valuation functions can be attacked.}
%     \label{fig:memattackdiagram}
% %\end{figure}
% \end{wrapfigure}
\citet{lai2020adversarial} provide an approach for adversarial robustness of general estimators (not influence estimators). \citet{yadav2024influencebasedattributionsmanipulated} address the robustness of influence functions on adversaries that control the valuation algorithm, a strong assumption. In contrast to these works, we present the {\em first theoretical analysis} for vulnerability of memorized-based valuation functions. We particularly assert that when a trained algorithm exhibits high accuracy, we can characterize highly memorized samples with highly sensitivity. We take advantage of the inverse operation's naturally high sensitivity to motivate a simple, computationally efficient, and effective Pseudoinverse attack to produce highly memorized samples in image classification tasks.
In comparison to prior work, we assume a much weaker adversarial model, one where the (malicious) data sellers only change the provided data, require no collusion, and are unaware of the exact (memorization-based) valuation algorithm. In our most effective attack, the adversary need not be aware of knowledge of the underlying data distribution! 
Figure~\ref{fig:memattackdiagram} illustrates the threat model. The simplicity, yet effectiveness, of both our theory and attack demonstrate a fundamental vulnerability in data valuation for supervised learning.
 %to adversarial corruption, and affirm the need for robust algorithms in future data markets. %we consider for adversarial collusion on memorization scores and their related proxies.

\begin{figure}
    \centering
    \includegraphics[width=0.6\linewidth]{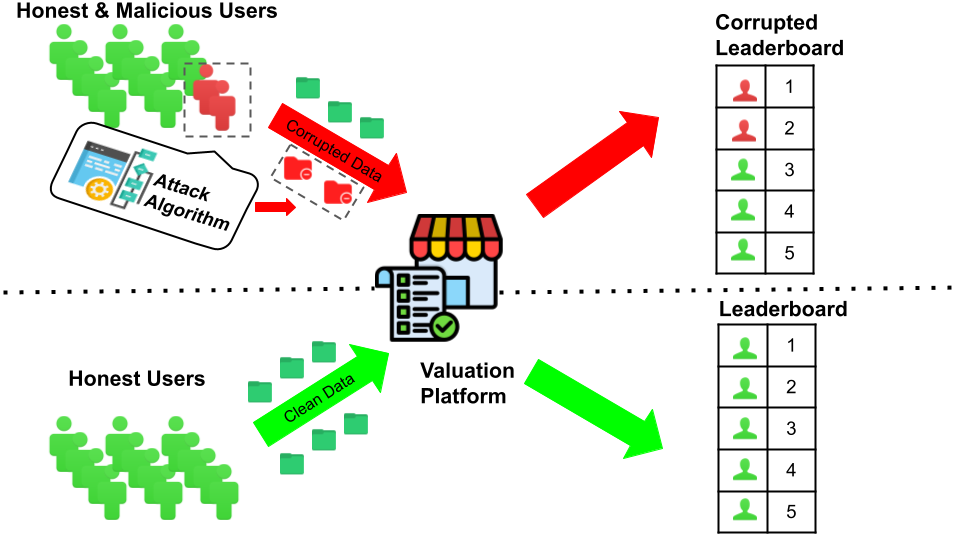}
    \caption{An overview of how data valuation functions can be attacked}
    \label{fig:memattackdiagram}
\end{figure}

We summarize our contributions below:
\begin{squishitemize}
\item We develop a theoretical framework for analyzing vulnerability of memorization scores to adversarial manipulation and demonstrate its implications on potential data pricing (\S~\ref{sec:approach}). 
\item We provide a simple and efficient Pseudoinverse attack that successfully alters memorization-based scoring of image classification tasks agnostic of underlying dataset and model architecture  (\S~\ref{sec:setup}).
\item We verify our theoretical guarantees with experiments on image classification tasks across various convolutional network architectures on the MNIST, SVHN, and CIFAR-10 datasets (\S~\ref{sec:results}), and present additional support over more complex transformer architecture, higher resolution image datasets, and other data modalities in (\S~\ref{sec:additional-experiments}).
\end{squishitemize}
\section{Related Work}
\label{sec:related-work}

\textbf{Data Markets:} \citet{dm:survey:reallife:azcoitia2022survey} give a survey on existing commercial data markets and their business models. \citet{pricing:mehta2021sell} introduce pricing policies for enabling data economies. \citet{dm:agarwal2019marketplace} introduce a mathematical model for data marketplaces. Also, recent work surveys privacy and security vulnerabilities in data markets and offers possible design solutions~\citep{AGMSW25}.
Our work introduces a realistic threat model and calls for robust algorithmic solutions for future data markets.

\textbf{Data Valuation:}
% Some recent studies explore how data attribution and data valuation functions can be adopted to facilitate dynamic systems in the data market~\citep{dm:pricing:zhao2023addressing} as well as IoT~\citep{payment:survey:saputhanthri2022survey} and database query~\citep{discovery:example:rezig2021dice} settings.
% % , while others consider scoring fairness in the decentralized learning~\citep{valuation:fl:fan2022improving} and adversarial training~\citep{yadav2024influencebasedattributionsmanipulated} setting.
% Ghorbani and Zou~\citep{ghorbani2019data} first introduce Data Shapley as a metric to quantify each individual data points marginal contribution to predictor performance. Feldman and Zhang~\citep{feldman2020neural} introduce empirical influence scoring as a means of quantifying the leverage of training data from conditional probability distributions. However these formulations require training thousands of models for accuracy, and are computationally infeasible for a real-time data market setting. \citet{park2023trak} leverages low-rank projections to give a fast approximation of influence while cutting down on the number of models required for scoring accuracy. Koh and Liang~\citep{koh2017understanding} explore a classical technique from robust statistics in influence functions to quantify a Hessian-based influence score. The memorization proxies we consider in our work only require a single training run while maintaining high correlation with the more expensive label memorization score.
Shapley values are a classical concept in game theory which were first employed for the problem of data valuation in machine learning by~\citet{ghorbani2019data}, termed Data Shapley Value, which aims to measure the influence of individual data samples on some specified performance metric. Due to computational overhead of computing Shapley Value, various influence functions have been proposed and adopted, including finding a core set of data examples~\citep{valuation:yan2021if}, gradient-based approximations to the Shapley Value~\citep{valuation:fl:song2019profit}, adaptations to federated learning setting~\citep{valuation:fl:wang2020principled}, as well as non-Shapley based influence functions such as gradient tracing~\citep{pruthi2020estimating} and low-rank kernel approximation~\citep{park2023trak}. However, it is not always clear what target for which the influence of data samples should be measured on. Common settings involve a specific test set on which the effects of influence are measured, but a definitive target test set might not be available or might change over time, or might not even be known in certain applications. Label memorization, the main metric of data valuation in this study, is a measure of self-influence that removes this particular challenge~\citep{feldman2020neural}. \citet{feldman2020does} shows that accuracy on training examples with label memorization scores is crucial to low generalization error, demonstrating that highly memorized samples strongly influence downstream tasks.

\textbf{Adversarial Attacks Against Memorization:} Privacy risk in the context of membership inference attack is closely related to memorization as shown previously by~\citet{choi2023train}.
\citet{carlini2022membership} find outliers to be more vulnerable to membership inference attack, whereas other prior work use poisoned images~\citep{jagielski2020auditing} or artificially crafted ``canaries" for privacy auditing of language models~\citep{carlini2019secret, thakkar2020understanding} to produce samples of high privacy risk, and subsequently high memorization scores. In comparison to prior adversarial studies, we present the first theoretical analysis characterizing highly memorized samples.

% \textbf{Robustness of data valuation:} 
% Basu et al.~\citep{basu2021influence} find that influence functions are sensitive to model architecture. Lai et al.~\citep{lai2020adversarial} characterize the robustness of general estimators to adversarial attacks.
% Yadav et al.~\citep{yadav2024influencebasedattributionsmanipulated} consider the setting of adversarial training for manipulation of influence scores, which presents a threat model much stronger than ours. To our knowledge, we present the first exploration in adversarial data corruption in the data market setting.

\section{Preliminaries and Notation}
\label{sec:prelims}

\subsection{Queries: Sensitivity and Accuracy}

Let $\univ$ be a universe or domain. e.g., $\univ = \reals^{3}$.
Also, let $\dist$ be a distribution over $\univ$ from which 
samples $\sample =
(\samp_1, \cdots, \samp_n) \in \univ^n$ can be drawn.
We use the notation $\samp_{1},\dots,\samp_{n} \getsr \dist$
to indicate that $\samp_{1},\dots,\samp_{n}$ is randomly
drawn from $\dist$.
A mechanism, trained by the data evaluator or data platform, can answer
\emph{queries}, from some query family $\queryset$, about $\dist$ or
$\sample$.
For any $\query\in\queryset$, we define the query answers on either
the population level or sample-level:
$$\query(\dist) = \ex{\sample \getsr \dist}{\query(\sample)} \eqand \query(\sample) = \frac{1}{n} \sum_{i\in [n]} \query(\samp_i).$$

Query families can be separated by their sensitivities, which
quantifies how much the query changes when one or more elements of
the input are changed.
Let $\sample\sim \sample'$ denote that $\sample, \sample' \in \univ^n$ differ on at most one
entry.  For any query $q\in\queryset$ and neighboring $\sample\sim \sample'$,
$\norm{q(\sample)-q(\sample')}$ (the $\ell_2$-norm of the difference between $q(\sample)$ and $q(\sample')$) 
is the sensitivity of the query $q:\univ^n\rightarrow\univ$. We can
define $\Delta$-sensitive queries:

\begin{definition}[$\Delta$-Sensitive Queries]
For $\Delta \geq 0$, $n\in \naturals$, 
these queries are specified by a function $\query \from
  \univ^n \to \univ$ where $\univ\subseteq \reals^*$.
  Further, the queries satisfy $$\norm{q(\sample)-q(\sample')} \leq \Delta,$$
  for every pair
  $\sample, \sample' \in \univ^n$ differing in only one entry.
\end{definition}

$\queryset_\Delta$ is the set of all queries with sensitivity
of at most $\Delta$.

As done in the literature  on stability and
adaptive data analysis~\citep{BNSSSU16, BousquetE02, Shalev-ShwartzSSS10},
$\mech$ is a stateful algorithm with access to samples 
$\samp_1,\dots,\samp_{n} \in \univ$.
We can define an accuracy game between a stateful \emph{adversary} $\adv$ and a mechanism $\mech$, illustrated in Figure~\ref{fig:accgame1}.
The (opposing) goal of
$\adv$ is to increase the error of the query answers
provided by $\mech$.
Because $\adv$ and $\mech$ are stateful, the queries and the
query answers may depend on the history of past queries and past 
query answers.

A primary goal of machine learning is to train an algorithm/mechanism $\mech$ that can accurately answer queries on $\dist$. The training can use (independent) samples $\samp_{1},\dots,\samp_{n} \getsr \dist$. 
Here, for each $i\in[n]$, $\samp_{i} =  (x_i, y_i)$ corresponds to
feature-label pairs, which are used to train $\mech$. Then a query for $\mech$
could be: \textit{given $\sample\in\univ^n$, what is the label $y_{n+1}$ for feature vector
$x_{n+1}$?} Clearly, this can be encoded via the query function $\query \from
  \univ^n \to \univ$ where $\query(\sample) = \query^*(\sample, x_{n+1})$ and
  $\query^*:\univ^n\times\univ\rightarrow\univ$ takes in samples and new
  feature vector  $x_{n+1}$.
The query answer is
$\samp_{n+1} = (x_{n+1}, y_{n+1})\in\univ$. For any $j\in[k]$, we measure query accuracy as
$\norm{a_j - q_j(\dist)}$ where $a_j$ is the answer by $\mech$---based on the samples
$\sample$ and  previous query answers---and
$q_j(\dist)$ is the query answer on the population level.
Note that the formalism goes beyond 
classification:
we might want to estimate the mean of a population in which case the query
is the population mean and the answer could be the mean of the samples
$\sample = (\samp_{1},\dots,\samp_{n})$.

\begin{figure}[ht!]
\begin{framed}
\begin{algorithmic}
\STATE{Sample $\samp_1,\dots,\samp_{n} \getsr \dist$ and let $\sample = (\samp_1,\dots,\samp_{n})$.}
\STATE{For $j = 1,\dots,\queries$}
\STATE{\quad$\adv$ outputs a query $\query_{j} \in \queryset$.}
\STATE{\quad$\mech(\sample, \{\query_{t}\}_{t=1}^j)$ outputs $a_j$.}
\end{algorithmic}
\end{framed}
\caption{The Accuracy Game $\accgame_{n, \queries, \queryset}[\mech, \adv]$ \label{fig:accgame1}}
\end{figure}

Since queries are meant to return answers, we measure how accurate the answers are either on the population or specific samples.

\begin{definition}[Population Accuracy] \label{def:accuratemechanism}
A mechanism $\mech$ is \emph{$(\alpha,\beta)$-accurate with respect to the population $\dist$ for $\queries$, potentially adaptively, chosen queries from $\queryset$ given $n$ samples in $\univ$} if for every adversary $\adv$,
$$
\pr{\accgame_{n, \queries, \queryset}[\mech, \adv]}{
\max_{ j \in [\queries]} ~ \norm{q_j(\dist) - a_j} \leq \alpha} \geq 1 - \beta.
$$
\end{definition}

% We can also measure accuracy relative to the samples given to the mechanism. The sample accuracy game is given in Figure~\ref{fig:sampleaccgame1}.

% \begin{figure}[ht!]
% \begin{framed}
% \begin{algorithmic}
% \STATE{$\adv$ chooses $\sample = (\samp_1,\dots,\samp_{n}) \in  \univ^n$.}
% \STATE{For $j = 1,\dots,\queries$}
% \STATE{\quad$\adv$ outputs a query $\query_{j} \in \queryset$.}
% \STATE{\quad$\mech(\sample, \{\query_{t}\}_{t=1}^j)$ outputs $a_j$.}
% \end{algorithmic}
% \end{framed}
% \caption{The Sample Accuracy Game $\sampaccgame_{n, \queries, \queryset}[\mech, \adv]$\label{fig:sampleaccgame1}}
% \end{figure}

% \begin{definition}[Sample Accuracy] \label{def:accuratemechanism-samp}
% A mechanism $\mech$ is \emph{$(\alpha,\beta)$-sample-accurate with respect to samples of size $n$ from $\univ$ for $\queries$ adaptively chosen queries from $\queryset$} if for every adversary $\adv$,
% $$
% \pr{\sampaccgame_{n, \queries, \queryset}[\mech, \adv]}{\max_{ j \in [\queries] }~~~ \norm{q_j(\sample)-a_j}  \leq \alpha} \geq 1-\beta.
% $$
% \end{definition}

In Definition~\ref{def:accuratemechanism}, the 
randomness is over algorithms $\mech, \adv$ and the distribution $\dist$.

\subsection{Stability Notions}

A variety of stability notions (such as Max-KL stability, KL stability,
and TV stability) capture how sensitive an algorithm is to
changes to its input~\citep{MAL-112}. In our work, we focus on Max-KL stability but
our theoretical results can be ported to other notions of stability:

\begin{definition}[Max-KL Stability]
\label{def:MKLstable}
Let $\mech \from \univ^n \to \calR$ be a randomized algorithm.  We say that $\mech$ is \emph{$(\eps, \delta)$-max-KL stable} if for every pair of samples $\sample, \sample'$ that differ on exactly one element, and every $R \subseteq \calR$, 
$$
\pr{}{\mech(\sample) \in R} \leq e^{\eps} \cdot \pr{}{\mech(\sample') \in R} + \delta.
$$
\end{definition}

For any $\eps, \delta\geq 0$, it is easy to see that
Max-KL stability is equivalent to
\emph{$(\eps, \delta)$-differential privacy}~\cite{DMNS06, DKMMN06}.

Post-processing refers to the notion that a property (e.g., stability)
is preserved after certain modifications.
All the stability notions are preserved under 
post-processing:

\begin{lemma}[Post-Processing of Stability Notions (e.g., see~\citep{bun2016concentrated})] \label{lem:postprocessing_gen}
Let $\mech \from \univ^n \to \calR$ and $f \from \calR \to \calR'$ be a pair of randomized algorithms.  If $\mech$ is \{$\eps$-TV, $\eps$-KL, $(\eps, \delta)$-max-KL\}-stable then the algorithm $f(\mech(\sample))$ is \{$\eps$-TV, $\eps$-KL, $(\eps, \delta)$-max-KL\}-stable.
\end{lemma}

In addition, composition of disjoint databases preserves
the stability notions:
\begin{lemma}[Post-Processing of Stability Notions (e.g., see~\citep{bun2016concentrated})] \label{lem:comp_gen}
Let $\mech_1 \from \univ^n \to \calR$ and $\mech_2 \from \univ^n \to \calR$ be a pair of randomized algorithms.  If $\mech_1, \mech_2$ are \{$\eps$-TV, $\eps$-KL, $(\eps, \delta)$-max-KL\}-stable then $(\mech_1(\sample_1), \mech_2(\sample_2))$ is \{$\eps$-TV, $\eps$-KL, $(\eps, \delta)$-max-KL\}-stable if $\sample_1, \sample_2$ are
disjoint.
\end{lemma}

\subsection{Memorization Scores \& Proxies}
\label{sec:proxies}

Memorization reflects how much a model relies on memorizing specific training examples instead of generalizing.

\begin{definition}[Label Memorization~\citep{feldman2020neural}]
For training point \( z_i = (x_i, y_i) \), the memorization score is:
\[
\mem(\mathcal{A}, \mathbf{z}, z_i) := \Pr_{h \leftarrow \mathcal{A}(\mathbf{z})}[h(x_i) = y_i] - \Pr_{h \leftarrow \mathcal{A}(\mathbf{z}^{\setminus z_i})}[h(x_i) = y_i].
\]
\label{def:feldman}
\end{definition}
Here, $h$ is the (randomized) classifier obtained from training algorithm $\mathcal{A}$ on dataset $\mathbf{z}$.

However, the first formulation of memorization proposed by Feldman and Zhang requires training thousands of models and is computationally prohibitive. \citet{zhao2024scalability} names two computationally efficient proxies with empirically high correlation with memorization that we describe below to use in our experiments. We describe an additional proxy per risk scores computed from a membership inference attack given by~\citet{song2021systematic}.

\noindent{\bf 1. Input Loss Curvature-based Proxies:} \citet{garg2023memorization} observe that data samples with high loss curvature visually correspond to long-tailed, mislabeled or conflicting samples, which are more likely to be memorized.

% \varun{$\ell$ is not defined}
\begin{definition}[Loss Curvature]
% ~\cite{garg2023memorization}]
The curvature of the loss function with respect to an input \( x_i \) is:
\[
\curv(\ell,i) := \text{Tr}\left( \nabla^2_x \ell(h_\theta(x_i), y_i) \right).
\]
Implicitly, $\theta$ is a hypothesis trained from randomized algorithm $\mathcal{A}$, and $h_\theta$ is the model output from parameters $\theta$. $\ell$ is the loss function used to train models in $\mathcal{A}$, e.g. cross entropy.

\end{definition}

% Garg observes a high cosine similarity on highly scoring samples using the Feldman and Zhang and curvature metric. \varun{so?}

\vspace{1mm}
\noindent{\bf 2. Learning Events-based Proxies:} The learning event proxies are first introduced by \citet{jiang2020characterizing}, a class of memorization score proxies designed to measure how quickly and reliably a model learns a specific example during training. The key intuition is that a specific example that is consistent with many others should be learned quickly as the gradient steps for all consistent examples should be well aligned. \citet{zhao2024scalability} find a strong link between these learning events proxies and memorization scores per the Spearman correlation coefficient. As done by prior work~\citep{jiang2020characterizing, zhao2024scalability}, we aggregate the following cumulative statistics over each training epoch: confidence, max confidence, entropy and binary correctness. Full details of this class of attacks are described in Appendix \ref{sec:attack-details}.

\begin{definition}[Learning Events Proxy]
% ~\citep{jiang2020characterizing}]
Given some per-sample event function $\phi(h_\theta(x), y)$ that depends on the learning hypothesis $\theta$ i.e. confidence, we define the cumulative training event proxy:
$$\texttt{event}(\mathcal{A},i,\phi)=\frac{1}{T}\sum_{t=1}^T\phi(h_{\theta_t}(x_i), y_i).$$
Here $\theta_t$ denotes the hypothesis learned at epoch $t$ in training algorithm $\mathcal{A}$. 
\end{definition}

\vspace{1mm}
\noindent{\bf 3. Membership Inference Attacks:} Membership Inference is a topic highly related with memorization. The goal of a membership inference attack is to recover whether a particular data entry was part of an unknown training set, either by using knowledge of the model, access to the model, or in a black-box setting. \citet{choi2023train} previously shows a theoretical link between Membership Inference advantage to memorization. Using a model trained on a certain dataset as the target model in an attack, we obtain risk scores for each data entry in the dataset---the probability of inclusion in the dataset---as a proxy for memorization scoring. The intuition comes from observing that highly memorized data entries are unlike representative data entries, and have higher probability of being identified by black-box model access. \citet{song2021systematic} give formulation for privacy risk score.

\begin{definition}[Privacy Risk Score]
% ~\citep{song2021systematic}]
    The privacy risk score $r$ of an input sample $z=(x,y)$ for the target machine learning model $h\leftarrow\mathcal{A}(\mathbf{z})$ is defined as the posterior probability that it is from the (random) set $\mathbf{z}$ after observing the target model's behavior over that sample denoted as $O(h, z)$, i.e.
    \[
        r(z)=\Pr(z\in\mathbf{z} \;\vert\; O(h,z)).
    \]
\end{definition}

These proxy scores are independent of each other, and each proxy operates on a different scale than the other, so the raw numbers are not meaningfully comparable. We give results comparing relative scoring in Appendix~\ref{sec:appendix-percentile}.

\section{Manipulating Memorization Scores}
\label{sec:approach}

We formally analyze the conditions under which data sellers can
manipulate memorization scores. The contrapositive of Theorem~\ref{thm:sensitive} is that
if the algorithm that is used to compute the memorization score is sufficiently
accurate, then the memorization will be high for a family of queries.
Our work builds on the theory of
stability notions in the literature~\citep{MAL-112, BNSSSU16, BousquetE02, Shalev-ShwartzSSS10}.

For our analysis, we note that \textit{we do not control how the algorithm $\mech$ is trained}! Thus, our adversarial
model is natural and affords the algorithm designer the ability to
respond to collusion by data sellers by modifying how
$\mech$ operates on the dataset.

\begin{tcolorbox}
We show that if the classification algorithm $\calA$ is 
very accurate on the population, then there \textit{always}
exist new examples from the data sellers that will lead to high
memorization scores.
\end{tcolorbox}

Let $q:\univ^n\rightarrow\univ$ be a query function. For any
fixed dataset $\sample$, the goal of the adversarial 
data seller is to consider if either $q(\sample)\cup \sample$ or $\sample$ leads to a
higher valuation score (via memorization).

For any algorithm $\mech$, we study the following question:
\textit{what queries on the dataset would lead to high memorization scores}?
In order to quantify the question,
we consider the following memorization score on addition of a new
example to the existing dataset:
\begin{equation}
    \mem(\mathcal{A}, \sample,q(\sample))\eqdef \Pr_{(x, y)\leftarrow q(\sample), h\leftarrow\mathcal{A}(\sample\cup q(\sample))}[h(x)=y]-\Pr_{(x, y)\leftarrow q(\sample), h\leftarrow\mech(\sample)}[h(x)=y]
    \label{eq:mem}
\end{equation}

In Equation~\ref{eq:mem}, what query functions
$q:\univ^n\rightarrow\univ$ would lead to high memorization?
We can measure sensitivity of a query
as $\max_{\sample, \sample'} \norm{q(\sample)-q(\sample')}$.

\begin{theorem}[See Theorem~\ref{thm:app-sensitive} in Appendix]
Let $Q_\Delta$ be a family of $\Delta$-sensitive queries on $\univ^n$.
Fix $\delta\in[0, 1]$ and let 
the dataset size $n\in\naturals$ be such that
there exists $\gamma > \delta$ such that $n \geq 1/\gamma$.
Then for any $\alpha, \beta\in (0, 1/10)$, there exists
algorithm $\mech$ with memorization score 
(i.e., $\mem(\mech, \sample, q(\sample)) \leq \delta$ from Equation~\ref{eq:mem}) of at most
$\delta$ such that $\mech$ is
$(\alpha, \beta)$-accurate for any query from $Q_\Delta$ but
it must be the case that
$\alpha \geq \gamma\Delta n$ and $\beta \geq \frac{\delta}{2\gamma}$.

That is, there  exists algorithm with
memorization score (Equation~\ref{eq:mem}) of at most $\delta$
such that for $\sample\getsr\dist^n$ and query $\query\in Q_\Delta$,
$\Pr[\norm{\query(\sample) - \query(\dist)} \geq \gamma\Delta n]\geq\frac{\delta}{2\gamma},$
where $\dist$ is a distribution over $\univ$.

\label{thm:sensitive}
\end{theorem}

The inverse query $q(z)=z^{-1}$ has sensitivity that approaches $\infty$ (i.e.,
let the input be non-invertible or be 0 for one-dimensional input). This motivates a subset of
our attacks (i.e., taking inverses of one or more examples).
More generally,
Theorem~\ref{thm:sensitive} implies that
in order to avoid manipulation by data sellers, the algorithm $\mech$ cannot
be too accurate on the population level or the algorithm $\mech$ must
(formally) satisfy stability guarantees, such as max-KL stability.

\paragraph{Full Details}
For the full details of our proofs and analysis, see Appendix~\ref{sec:theory}.
\section{Experimental Setup}
\label{sec:setup}

%A more detailed description is presented in Appendix~\ref{sec:additional-experiments}. 
    % \daniel{fix all the broken links.}

\vspace{1mm}\noindent{\bf Attacks:}
We evaluate the robustness of data valuation methods under four distinct input-space attacks, each modifying input \( x \) while preserving the label \( y \). The attacks are presented from low to high fidelity i.e., how closely the adversarially modified image represents the original. Previous work~\citep{xing2023you} demonstrates that low-fidelity images can have high utility in synthetic data, which justifies the consideration of our visually aberrated attack images. Our attacks are motivated by distributional shift, stability notions and decision boundary proximity. A detailed description of how these attacks are implemented is given in Appendix~\ref{sec:attack-details}.

\begin{squishenumerate}
    \item {\em No Attack:} Henceforth referred to as \texttt{None}, we use this label in our results section to denote unperturbed natural data samples.
    \item {\em Out-of-Distribution Replacement Attack:} Henceforth referred to as \texttt{OOD}, this attack replaces in-distribution samples with out-of-distribution inputs from related datasets (e.g., FashionMNIST, CIFAR-10), testing sensitivity to semantic shifts. Intuitively, a small spurt of data drawn from a separate latent distribution will be long-tailed and memorized.
    
    \item {\em Pseudoinverse Attack:} Henceforth referred to as \texttt{PINV}, this attack transforms input images by computing their Moore-Penrose pseudoinverses and normalizing to image-space ranges, inducing unnatural but structured distortions. Intuitively, by treating images as information signals for a model to parse, prior work~\citep{brown2021memorization} suggests that there
    are natural tasks for which any high-accuracy algorithm  would need  to
    memorize the majority of the samples. Motivated by our theoretical results, we take the pseudoinverse of an image matrix in order to result in memorization of the new image.
    
    \item {\em Naive EMD Attack:} Henceforth referred to as \texttt{EMD}, this attack maximizes the Wasserstein distance between original and perturbed images using a greedy per-pixel binary search heuristic over RGB intensities. Intuitively, we seek to maximize the distance between the original and perturbed image, treating image data as probability distributions.

    \item {\em DeepFool Perturbation Attack:} Henceforth referred to as \texttt{DF}, this attack applies the DeepFool algorithm~\citep{moosavi2016deepfool, abroshan2024superdeepfool} to perturb each input minimally toward the classifier's decision boundary, forcing misclassification. Intuitively, we might expect a data point close to the decision boundary to be memorized by the model.
\end{squishenumerate}

\vspace{1mm}\noindent{\bf Datasets:}
We conduct experiments on three canonical image classification datasets: MNIST~\citep{lecun1998gradient}, SVHN~\citep{netzer2011reading}, and CIFAR-10~\citep{krizhevsky2009learning}. These datasets span grayscale handwritten digits (MNIST), real-world digit photographs (SVHN), and natural scene object categories (CIFAR-10), and thus collectively test attribution methods across increasing levels of input complexity and semantic variation. We also present supplemental results over ImageNet~\citep{russakovsky2015imagenet} and AG News~\citep{zhang2015character} in Appendix~\ref{sec:appendix-vit}.

\vspace{1mm}\noindent{\bf Models:}
We evaluate three standard deep neural net architectures: VGG-11~\citep{simonyan2014very}, ResNet-18~\citep{he2016deep}, and MobileNet-v2~\citep{sandler2018mobilenetv2}. These models represent different design paradigms: convolutional (VGG), residual (ResNet), and mobile-efficient (MobileNet). Each model is trained using standard data augmentation and optimization techniques (SGD with cross-entropy loss), with architectures matched to dataset resolution where applicable. We also present supplemental results over Vision Transformer (ViT)~\citep{dosovitskiy2020image} and BERT~\citep{devlin2019bert} architecture in Appendix~\ref{sec:appendix-vit}.

\vspace{1mm}\noindent{\bf Experimental Setting:}
We evaluate on the three proxies described in \S~\ref{sec:proxies}.
A detailed description of scoring implementation is found in Appendix~\ref{sec:additional-experiments}. Each attack is evaluated over \( t = 5 \) independent trials. In each trial, a randomly chosen subset, which we call the \textit{attack set}, of data points from the base dataset is perturbed using the attack under consideration. We consider 3 sizes of attack sets: 10, 100, 1000. We then average the valuation scores of the perturbed samples and report the mean over trials as the final {\em attack score}.

\vspace{1mm}\noindent{\bf Hardware Used:}
All experiments were conducted using NVIDIA A40 single GPU nodes. Training and attack procedures were implemented in PyTorch. More details can be found in Appendix~\ref{sec:additional-experiments}.

\section{Experimental Results}
\label{sec:results}

We aim to answer the following questions: (1) What attack is the most effective in manipulating memorization scores?; (2) Is this attack effective across dataset and model architecture settings?

As a quick summary, we empirically validate the following claims: (1) The high sensitivity queries produced by \texttt{PINV} perturb memorization scores (and proxies) to a greater extent than other attacks we consider; (2) \texttt{PINV} outperforms other non-motivated attacks across datasets and model architectures, shown through extensive experimentation over the MNIST, SVHN and CIFAR-10 datasets using convolutional network architectures and limited results over higher resolution datasets and textual data over ImageNet and AG News dataset using transformer architectures.

We also find that as the size of the attack set reaches $10^4$, \texttt{PINV} appears to lose advantage over \texttt{OOD} and \texttt{EMD}. However \texttt{PINV} does not require knowledge of the underlying data distribution (in comparison to \texttt{OOD}) and remains more computationally efficient (in comparison to \texttt{EMD}). While one might wonder if such adversarial modifications degrade model generalization, we find that the test accuracy does not significantly decrease; full testing analysis is presented in Appendix~\ref{sec:additional-experiments}.

Memorization scoring and its proxies are inherently metrics which have dependence on both individual data samples and underlying dataset. To compare performance the performance of different attacks, we present our results in accordance to \textit{expected attack advantage} (EAA), our defined metric for comparing the memorization scoring perturbation capabilities for the attacks considered in our study. 

In particular, EAA captures the expected improvement of memorization scoring between an adversary's attack in comparison to their honest uncorrupted (no attack) data as a baseline, motivated by our threat model of an adversary without knowledge of underlying data distribution. \textit{Concretely, EAA captures the expected difference between scoring of data samples that are attacked and data samples that are produced by \texttt{None}}.

\noindent{\bf Note:} 
Appendix~\ref{sec:additional-experiments} contains more visualizations and results over different experimental settings, including higher resolution datasets, textual data and transformer architectures.

\subsection{Measuring Attack Effectiveness}

\newcommand{\pp}{$\pm$}

% \begin{table}[h!]
% \centering
% \resizebox{\textwidth}{!}{%
% \begin{tabular}{c|ccc|ccc|ccc}
% \toprule
% \textbf{Attack} & \multicolumn{3}{c|}{\textbf{Loss Curvature}} & \multicolumn{3}{c|}{\textbf{Confidence Event}} & \multicolumn{3}{c}{\textbf{Privacy Score}} \\
% \cmidrule(r){2-4} \cmidrule(r){5-7} \cmidrule(r){8-10}
% & MNIST & SVHN & CIFAR-10 & MNIST & SVHN & CIFAR-10 & MNIST & SVHN & CIFAR-10 \\
% \midrule
% \texttt{None} & 0.01\pp0.00 & 0.01\pp0.00 & 0.06\pp0.00 & 0.01\pp0.00 & 0.07\pp0.00 & 0.24\pp0.00 & 0.50\pp0.00 & 0.49\pp0.00 & 0.30\pp0.00 \\
% \midrule
% \texttt{OOD} & 0.07\pp0.00 & 0.01\pp0.00 & 0.25\pp0.00 & 0.80\pp0.00 & 0.61\pp0.00 & 0.63\pp0.00 & 0.36\pp0.00 & 0.06\pp0.00 & 0.49\pp0.00  \\
% \texttt{PINV} & 0.08\pp0.00 & \textbf{0.23\pp0.00} & \textbf{0.37\pp0.00} & \textbf{0.85\pp0.00} & \textbf{0.81\pp0.00} & \textbf{0.66\pp0.00} & \textbf{0.43\pp0.00} & \textbf{0.51\pp0.00} & \textbf{0.70\pp0.00}  \\
% \texttt{EMD} & \textbf{0.12\pp0.00} & 0.04\pp0.00 & 0.22\pp0.00 & 0.84\pp0.00 & 0.74\pp0.00 & 0.59\pp0.00 & 0.31\pp0.00 & 0.08\pp0.00 & 0.15\pp0.00 \\
% \texttt{DF} & 0.00\pp0.00 & 0.00\pp0.00 & 0.00\pp0.00 & 0.00\pp0.00 & 0.00\pp0.00 & 0.00\pp0.00 & 0.01\pp0.00 & 0.01\pp0.00 & 0.01\pp0.00 \\
% \bottomrule
% \end{tabular}%
% }
% \caption{\textbf{EAA on VGG-11 architecture with attack set size of 100.} \texttt{PINV} outperforms almost all other attacks by a significant margin across dataset and proxy.}
% \label{fig:vgg-100-table}
% \end{table}

\begin{table}[H]
\centering
\resizebox{\textwidth}{!}{%
\begin{tabular}{c|ccc|ccc|ccc}
\toprule
\textbf{Attack} & \multicolumn{3}{c|}{\textbf{Loss Curvature}} & \multicolumn{3}{c|}{\textbf{Confidence Event}} & \multicolumn{3}{c}{\textbf{Privacy Score}} \\
\cmidrule(r){2-4} \cmidrule(r){5-7} \cmidrule(r){8-10}
& MNIST & SVHN & CIFAR-10 & MNIST & SVHN & CIFAR-10 & MNIST & SVHN & CIFAR-10 \\
\midrule
\texttt{None} & 0.00\pp0.00 & 0.01\pp0.00 & 0.09\pp0.00 & 0.01\pp0.00 & 0.06\pp0.00 & 0.22\pp0.00 & 0.45\pp0.00 & 0.49\pp0.00 & 0.18\pp0.00 \\
\midrule
\texttt{OOD} & 0.15\pp0.00 & 0.02\pp0.00 & 0.20\pp0.00 & 0.51\pp0.00 & 0.48\pp0.00 & 0.58\pp0.01 & 0.06\pp0.01 & -0.12\pp0.00 & 0.10\pp0.00  \\
\texttt{PINV} & \textbf{0.20\pp0.00} & \textbf{0.35\pp0.00} & \textbf{0.25\pp0.00} & \textbf{0.67\pp0.01} & \textbf{0.79\pp0.00} & \textbf{0.64\pp0.00} & \textbf{0.30\pp0.02} & \textbf{0.46\pp0.00} & \textbf{0.58\pp0.01}  \\
\texttt{EMD} & 0.07\pp0.00 & 0.00\pp0.00 & -0.05\pp0.00 & 0.33\pp0.01 & 0.59\pp0.00 & 0.39\pp0.00 & -0.08\pp0.02 & -0.05\pp0.00 & -0.05\pp0.00 \\
\texttt{DF} & 0.00\pp0.00 & 0.01\pp0.00 & 0.00\pp0.00 & -0.01\pp0.00 & 0.01\pp0.00 & -0.03\pp0.00 & -0.02\pp0.01 & -0.03\pp0.00 & -0.02\pp0.00 \\
\bottomrule
\end{tabular}%
}
\caption{\textbf{EAA on ResNet-18 architecture with attack set size of 10.} \texttt{PINV} outperforms all other attacks by a significant margin across dataset and proxy.}
\label{fig:resnet-10-table}
\end{table}

\begin{table}[h!]
\centering
\resizebox{\textwidth}{!}{%
\begin{tabular}{c|ccc|ccc|ccc}
\toprule
\textbf{Attack} & \multicolumn{3}{c|}{\textbf{Loss Curvature}} & \multicolumn{3}{c|}{\textbf{Confidence Event}} & \multicolumn{3}{c}{\textbf{Privacy Score}} \\
\cmidrule(r){2-4} \cmidrule(r){5-7} \cmidrule(r){8-10}
& MNIST & SVHN & CIFAR-10 & MNIST & SVHN & CIFAR-10 & MNIST & SVHN & CIFAR-10 \\
\midrule
\texttt{None} & 0.00\pp0.00 & 0.01\pp0.00 & 0.09\pp0.00 & 0.01\pp0.00 & 0.06\pp0.00 & 0.23\pp0.00 & 0.47\pp0.00 & 0.49\pp0.00 & 0.19\pp0.00 \\
\midrule
\texttt{OOD} & 0.13\pp0.00 & 0.02\pp0.00 & \textbf{0.14\pp0.00} & 0.62\pp0.00 & 0.52\pp0.00 & 0.61\pp0.00 & 0.08\pp0.00 & -0.10\pp0.00 & 0.09\pp0.00  \\
\texttt{PINV} & \textbf{0.14\pp.00} & \textbf{0.14\pp0.00} & 0.08\pp0.00 & \textbf{0.85\pp0.00} & \textbf{0.81\pp0.00} & \textbf{0.66\pp0.00} & \textbf{0.29\pp0.01} & \textbf{0.51\pp0.00} & \textbf{0.79\pp0.00}  \\
\texttt{EMD} & 0.06\pp0.00 & 0.00\pp0.00 & -0.05\pp0.00 & 0.51\pp0.00 & 0.68\pp0.00 & 0.54\pp0.00 & -0.03\pp0.00 & -0.03\pp0.00 & 0.01\pp0.00 \\
\texttt{DF} & 0.00\pp0.00 & 0.00\pp0.00 & 0.01\pp0.00 & 0.00\pp0.00 & 0.00\pp0.00 & 0.00\pp0.00 & -0.02\pp0.00 & -0.01\pp0.00 & -0.02\pp0.00 \\
\bottomrule
\end{tabular}%
}
\caption{\textbf{EAA on ResNet-18 architecture with attack set of size 100.} \texttt{PINV} outperforms almost all other attacks by a significant margin across dataset and proxy.\vspace{-5mm}}
\label{fig:resnet-100-table}
\end{table}

\vspace{1mm}
Across our experimental settings, \texttt{PINV} outperforms the other attacks and produces samples that score highly relative to the base (unmodified) dataset. 
Tables~\ref{fig:resnet-10-table} \& \ref{fig:resnet-100-table} record the effectiveness of our four attacks on ResNet-18 models across difference choices of datasets and memorization scores and proxies (for an attack set sizes of 10 and 100). One can see the clear advantage of our theoretically motivated attack \texttt{PINV} in comparison to the other unmotivated attacks considered. 
% Table~\ref{fig:resnet-100-table} repeats the analysis, and shows the effectiveness on ResNet-18 models instead. %Each attack row tracks the induced change in score from uncorrupted datasets. 

\noindent{\bf An Explanation:} Our theoretical work suggests taking the inverse of an in-distribution sample strongly distinguishes it from the rest of the underlying dataset. We assert that it is the high sensitivity that causes the strongest scoring increase; additional experiments in Appendix~\ref{sec:additional-experiments} favorably compare \texttt{PINV} to random noise. On the other hand, while \texttt{DF} produces samples that are not meaningfully distinguishable from the underlying dataset, the attack is largely ineffective in perturbing memorization scoring. Based on Definition~\ref{def:feldman}, one might believe that points that are most likely to change the decision boundary (i.e., those that lie close to it) would be more memorized. Thus, we also consider attack sets comprising of points that lie close to the decision boundary. Results of the effectiveness of \texttt{DF} across different attack sets is presented in Appendix~\ref{sec:additional-experiments}. We find that boundary starting points marginally improves the performance of \texttt{DF}, but not enough to beat out \texttt{PINV} which does not assume any prior knowledge of the underlying dataset.
In line with our expectations, \texttt{OOD} performs well relative to baseline scoring. However, its performance compared with our \texttt{PINV} suggests that there is an underlying informational component to memorization that is not fully captured by distributional shifts.

\subsection{Label Memorization}

For completeness, we also provide a limited set of experiments for the effectiveness of the \texttt{OOD} and \texttt{PINV} attack on the more computationally intensive label memorization scoring per~\cite{feldman2020neural}. Unlike Feldman \& Zhang, we only use $n=100$ number of models trained per scoring run due to computational constraints. Each model is trained using the same process as described in Appendix~\ref{sec:additional-experiments}. We demonstrate that the high sensitivity queries produced by the \texttt{PINV} attack raise label memorization scoring in accordance with our theoretical claims.

\begin{table}[H]
\centering
\resizebox{\textwidth}{!}{%
\begin{tabular}{c|ccc|ccc|ccc}
\toprule
\textbf{Attack} & \multicolumn{3}{c|}{\textbf{MNIST}} & \multicolumn{3}{c|}{\textbf{SVHN}} & \multicolumn{3}{c}{\textbf{CIFAR10}} \\
\cmidrule(r){2-4} \cmidrule(r){5-7} \cmidrule(r){8-10}
Attack set size & 10 & 100 & 1000 & 10 & 100 & 1000 & 10 & 100 & 1000 \\
\midrule
\texttt{None} & 0.01\pp 0.00 & 0.01\pp 0.00 & 0.01\pp 0.00 & 0.08\pp 0.00 & 0.08\pp 0.00 & 0.08\pp 0.00 & 0.23\pp 0.00 & 0.23\pp 0.00 & 0.23\pp 0.00 \\
\midrule
\texttt{OOD} & 0.08\pp 0.00 & 0.02\pp 0.00 & 0.01\pp 0.00 & 0.04\pp 0.00 & 0.02\pp 0.00 & 0.01\pp 0.00 & 0.02\pp 0.00 & 0.02\pp 0.00 & 0.01\pp 0.00 \\
\texttt{PINV} & \textbf{0.62\pp 0.00} & \textbf{0.36\pp 0.00} & \textbf{0.20\pp 0.00} &\textbf{ 0.58\pp 0.00} & \textbf{0.45\pp 0.00} & \textbf{0.13\pp 0.00} &\textbf{ 0.43\pp 0.01} & \textbf{0.23\pp 0.00} & \textbf{0.09\pp 0.00} \\
\bottomrule
\end{tabular}%
}
\caption{\textbf{EAA for label memorization across dataset over VGG-11 architecture.} \texttt{PINV} raises label memorization score significantly over base scoring and \texttt{OOD} attack.\vspace{-5mm}}
\label{fig:label-memorization-vgg}
\end{table}

\begin{table}[H]
\centering
\resizebox{\textwidth}{!}{%
\begin{tabular}{c|ccc|ccc|ccc}
\toprule
\textbf{Attack} & \multicolumn{3}{c|}{\textbf{MNIST}} & \multicolumn{3}{c|}{\textbf{SVHN}} & \multicolumn{3}{c}{\textbf{CIFAR10}} \\
\cmidrule(r){2-4} \cmidrule(r){5-7} \cmidrule(r){8-10}
Attack set size & 10 & 100 & 1000 & 10 & 100 & 1000 & 10 & 100 & 1000 \\
\midrule
\texttt{None} & 0.01\pp 0.00 & 0.01\pp 0.00 & 0.01\pp 0.00 & 0.07\pp 0.00 & 0.07\pp 0.00 & 0.07\pp 0.00 & 0.10\pp 0.00 & 0.10\pp 0.00 & 0.11\pp 0.00 \\
\midrule
\texttt{OOD} & 0.06\pp 0.00 & 0.02\pp 0.00 & 0.01\pp 0.00 & 0.06\pp 0.00 & 0.01\pp 0.00 & 0.01\pp 0.00 & 0.03\pp 0.00 & 0.02\pp 0.00 & 0.01\pp 0.00 \\
\texttt{PINV} & \textbf{0.51\pp 0.00} & \textbf{0.36\pp 0.00} & \textbf{0.17\pp 0.00} & \textbf{0.47\pp 0.00} & \textbf{0.30\pp 0.00} & \textbf{0.04\pp 0.00} & \textbf{0.36\pp 0.01} & \textbf{0.20\pp 0.00} & \textbf{0.11\pp 0.00} \\
\bottomrule
\end{tabular}%
}
\caption{\textbf{EAA for label memorization across dataset over ResNet-18 architecture.} \texttt{PINV} raises label memorization score significantly over base scoring and \texttt{OOD} attack.\vspace{-5mm}}
\label{fig:label-memorization-resnet}
\end{table}

\begin{table}[H]
\centering
\resizebox{\textwidth}{!}{%
\begin{tabular}{c|ccc|ccc|ccc}
\toprule
\textbf{Attack} & \multicolumn{3}{c|}{\textbf{MNIST}} & \multicolumn{3}{c|}{\textbf{SVHN}} & \multicolumn{3}{c}{\textbf{CIFAR10}} \\
\cmidrule(r){2-4} \cmidrule(r){5-7} \cmidrule(r){8-10}
Attack set size & 10 & 100 & 1000 & 10 & 100 & 1000 & 10 & 100 & 1000 \\
\midrule
\texttt{None} & 0.01\pp 0.00 & 0.01\pp 0.00 & 0.01\pp 0.00 & 0.06\pp 0.00 & 0.06\pp 0.00 & 0.06\pp 0.00 & 0.16\pp 0.00 & 0.16\pp 0.00 & 0.16\pp 0.00 \\
\midrule
\texttt{OOD} & 0.04\pp 0.00 & 0.01\pp 0.00 & 0.01\pp 0.00 & 0.03\pp 0.00 & 0.02\pp 0.00 & 0.01\pp 0.00 & 0.01\pp 0.00 & 0.01\pp 0.00 & 0.01\pp 0.00 \\
\texttt{PINV} & \textbf{0.34\pp 0.00} & \textbf{0.34\pp 0.00} & \textbf{0.19\pp 0.00} & \textbf{0.20\pp 0.00} & \textbf{0.12\pp 0.00} & \textbf{0.04\pp 0.00} & \textbf{0.32\pp 0.01} & \textbf{0.18\pp 0.00} & \textbf{0.10\pp 0.0}0 \\
\bottomrule
\end{tabular}%
}
\caption{\textbf{EAA for label memorization across dataset over MobileNet-v2 architecture.} \texttt{PINV} raises label memorization score significantly over base scoring and \texttt{OOD} attack.\vspace{-5mm}}
\label{fig:label-memorization-mobile}
\end{table}

\subsection{Semantic Meaning}

It can be argued that images produced by the \texttt{PINV} attack have low semantic meaning and consequently will not be realistically considered in a data market setting after manual inspection. We respectfully believe that manual inspection alone provides limited insight into the actual contribution of a sample toward model performance and serves as a poor baseline for detecting adversarial samples. As shown in \citet{wang2018dataset}, there exist many training samples that may appear semantically uninformative or even irrelevant to a human observer (i.e., easy to detect and discard) but nonetheless prove useful for generalization. Thus, it would be premature or misleading to discard or downweight such samples solely based on visual or semantic intuition.

Additionally, in many settings where data valuation is most critical e.g., active data curation, responsible dataset pruning, or cost-sensitive training, it is likely that the model's performance depends on the cumulative influence of a large number of individually subtle or redundant examples. This further reduces the practical viability of relying on human judgment alone. Manual inspection does not scale, especially when dealing with high-volume datasets or when the goal is to precisely quantify marginal utility of each point for fine-grained optimization.

\subsection{Transformer Architecture}

While we focus on low resolution datasets and convolutional network architectures due to computational constraints of training models for our experiments, we assert that our main theory is agonistic of training algorithm and model architecture. We present a limited set of experiments over ImageNet data using ViT architectures in Appendix~\ref{sec:appendix-vit} showing that \texttt{PINV} efficiently manipulates memorization scores for any training algorithm. We also present a small experiment over text classification dataset AG News over BERT transformer architecture that demonstrates the effectiveness of inverse-based attacks across other data modalities such as text in Appendix~\ref{sec:appendix-bert}.

\section{Conclusion}
\label{sec:conclusion}

In our work we present a theoretical justification for our framework of memorization score manipulation, as well as a suite of experimental results that demonstrate the vulnerability of memorization scores and its proxies to adversarial manipulation. We empirically find that a simple and efficient attack of taking a scaled Pseudoinverse of in-distribution image data is sufficient to successfully produce images with high memorization scores. We also give a theoretical starting point for further analysis of adversarial manipulation of memorization scores and other data attribution methods.

% \input{contents/limitations}
% \input{appendix/ethics}

%%%%%%%%%%%%%%%%%%%%%%%%%%%%%%%%%%%%%%%%%%%%%%%%%%%%%%%%%%%%

\newpage
\bibliographystyle{unsrtnat}
\bibliography{references}

\clearpage
\appendix

\section*{Appendix}
% \input{contents/limitations}
% \newpage
\section{Theoretical Analysis}
\label{sec:theory}

Recall that the goal is to study the following question:
\textit{what queries on the dataset would lead to high memorization scores}?
In order to quantify the question,
we consider the following memorization score on addition of a new
example to the existing dataset:
\begin{equation}
    \mem(\mathcal{A}, \sample,q(\sample))\eqdef \Pr_{(x, y)\leftarrow q(\sample), h\leftarrow\mathcal{A}(\sample\cup q(\sample))}[h(x)=y]-\Pr_{(x, y)\leftarrow q(\sample), h\leftarrow\mech(\sample)}[h(x)=y]
    \label{eq:app-mem}
\end{equation}

It is natural to consider adding $k > 1$ examples so we also
consider the case where $k$ (adaptive)
queries on the existing dataset can be made:

\begin{align}
&\mem(\mathcal{A}, \sample, \{q_i(\sample)\}_{i=1}^k) \\
&\eqdef\Pr_{\{(x_i, y_i)\}_{i=1}^k\leftarrow \{q_i(\sample)\}_{i=1}^k, h\leftarrow\mathcal{A}(\sample\cup\{q_i(\sample)\}_{i=1}^k)}[\cap_{i=1}^k h(x_i)=y_i]\\
&-\Pr_{\{(x_i, y_i)\}_{i=1}^k\leftarrow \{q_i(\sample)\}_{i=1}^k, h\leftarrow\mech(\sample)}[\cap_{i=1}^k h(x_i)=y_i]
\end{align}

We begin by restating our main theorem:

\begin{theorem}
Let $Q_\Delta$ be a family of $\Delta$-sensitive queries on $\univ^n$.
Fix $\delta\in[0, 1]$ and let 
the dataset size $n\in\naturals$ be such that
there exists $\gamma > \delta$ such that $n \geq 1/\gamma$.

Then for any $\alpha, \beta\in (0, 1/10)$, there exists
algorithm $\mech$ with memorization score 
(i.e., $\mem(\mech, \sample, q(\sample)) \leq \delta$ from Equation~\ref{eq:mem}) of at most
$\delta$ such that $\mech$ is
$(\alpha, \beta)$-accurate for any query from $Q_\Delta$ but
it must be the case that
$\alpha \geq \gamma\Delta n$
and $\beta \geq \frac{\delta}{2\gamma}$.

That is, there  exists algorithm with
memorization score (Equation~\ref{eq:mem}) of at most $\delta$
such that for $\sample\getsr\dist^n$ and query $\query\in Q_\Delta$,
$$
\Pr[\norm{\query(\sample) - \query(\dist)} \geq \gamma\Delta n]\geq\frac{\delta}{2\gamma},
$$
where $\dist$ is a distribution over $\univ$.

\label{thm:app-sensitive}
\end{theorem}

Our theory suggests that, in order
to avoid manipulation by data sellers, the algorithm $\mech$ cannot
be too accurate on the population level or the algorithm $\mech$ must
(formally) satisfy stability guarantees. e.g., max-KL stability which
can be achieved via the use of differential privacy.
Without such guarantees, one can effectively attack memorization scores. 
% Motivated by our formal analysis, we design inverse-based attacks to exploit the use of sensitive queries to increase memorization score of the new examples.
% TUE CHANGED THE ABOVE PARAGRAPH !!!

\paragraph{High-level Idea for Proof of Theorem~\ref{thm:app-sensitive}}
We will engineer a $\Delta$-sensitive
query that counts (at scale $\Delta$) how many sample points satisfy
a \textit{sample-dependent} predicate $p$. The predicate $p$ itself is
produced by a \textit{stable} $(0, \delta)$-max-KL procedure. We will
use the predicate to construct $\Delta$-sensitive queries 
(i.e., adding or removing a single record can change the induced query
value by at most $\Delta$). Intuitively, when $p$ is chosen from the
sample $\sample$, the value $q(\sample)$ is large ($\geq \gamma\Delta n$)
whereas the population value $q(\dist)$ remains 0. Stability keeps the
\textit{memorization score} (the advantage of re-training on
an added/removed example) at most $\delta$. And we prove that the 
empirical accuracy of such algorithm cannot, with high probability,
be close to the population accuracy.

\begin{figure}[h]
\begin{framed}
\begin{algorithmic}
\STATE{\textbf{Input:} 
A database $\sample\in[0,1]^n$. We think of $\sample$ as $\lceil\frac{1}{\gamma}\rceil$ databases of size $\gamma n$ each: $\sample = (\sample_1,\dots,\sample_{\lceil 1/\gamma \rceil})$.}
\STATE{\quad For $1\leq i\leq\lceil 1/\gamma \rceil$, let $\hat{\sample}_i=\calB(\sample_i)$.}
\STATE{\quad Let $p:[0,1]\rightarrow\{0,1\}$ where $p(z)=1$ iff $\exists i$ s.t.\ $z\in\hat{\sample}_i$.}
\STATE{\quad Define $\query_p:[0,1]^n\rightarrow\reals$ where $\query_p(\sample)=\Delta\sum_{z\in\sample}p(z)$ (note that $\query_p$ is a $\Delta$-sensitive query).}
\STATE{\textbf{Output:} $\query_p$.}
\end{algorithmic}
\end{framed}
\caption{$(0, \delta)$ max-KL stable algorithm $\mech$. \label{fig:zero-delta}}
\end{figure}

\begin{proof}[Proof of Theorem~\ref{thm:app-sensitive}]

First, we start with the following algorithm $\calB$ that is
$(0, \delta)$-max-KL stable (Definition~\ref{def:MKLstable}). 
$\calB$ is defined as follows:
\begin{itemize}
\item For any
database element $e$ input to $\calB$, 
output $e$
with probability $\delta$.
\item Otherwise, with probability $1-\delta$, output the empty database $\emptyset$ that is independent of the input to $\calB$.
\end{itemize}

Clearly, $\calB$ satisfies the definition of
$(0, \delta)$-max-KL stability.

Now, given a dataset $\sample\in[0,1]^n$, we partition the dataset
into $\lceil 1/\gamma \rceil$ disjoint blocks 
$\sample^{(1)}, \dots, \sample^{(\lceil 1/\gamma \rceil)}$, each of size $\gamma n$ (possible since $n \ge 1/\gamma$). 

Consider the stable mechanism $\mech$ defined in Figure~\ref{fig:zero-delta}:
on each block $\sample^{(i)}$, run $\calB$ and obtain output $\hat{\sample}^{(i)}$. 
Define a predicate $p:[0, 1] \to \{0,1\}$ by
\[
p(u) = \ind\left\{u \in \bigcup_{i\in \lceil 1/\gamma\rceil}\{\hat{\sample}^{(i)}\}\right\}.
\]
Now define the $\Delta$-sensitive query
\[
q_p(\sample) \;=\; \Delta \cdot \sum_{j=1}^n p(\sample_j).
\]
Since changing one coordinate alters the sum by at most $\Delta$, $q_p$ is indeed $\Delta$-sensitive. 
By Lemma~\ref{lem:comp_gen}, $q_p$ inherits the $(0,\delta)$-max-KL stability of $\calB$. Thus, $\mech$ satisfies $(0,\delta)$-max-KL stability since
$\calB$ is applied to disjoint databases.

Now let $\sample = (\sample_1,\dots,\sample_{\lceil 1/\gamma \rceil})$ contain i.i.d. samples from $\univ$, and consider the execution of $\calB$ on $\sample$.
Note that the predicate $p$ evaluates to 1 only on a finite number of points from $[0,1]$, an infinite universe. 
Therefore, 
$$\query_p(\dist)=\Delta\cdot n \cdot \pr{u\gets \dist}{p(z)=1} = 0.$$
Next, notice that $$\query_p(\sample)=\gamma\Delta n\cdot|\{ i : \hat{\sample}_i=\sample_i \}|.$$
Therefore, if there exists an $i$ s.t.\ $\hat{\sample}_i=\sample_i$ then $\query_p(\sample)- \query_p(\dist)\geq\gamma\Delta n$.
Since the probability of sampling any element in the database
is $\delta$, the probability that
$\query_p(\sample)- \query_p(\dist) <\gamma\Delta n$
is at most
$$
(1-\delta)^{1/\gamma}\leq e^{-\delta/\gamma} \leq 1-\frac{\delta}{2\gamma},
$$
for some $\gamma > \delta$.
Then, with probability at least $\frac{\delta}{2\gamma}$, algorithm $\mech$ outputs a $\Delta$-sensitive query $\query_p$ s.t.\ $\query_p(\sample)-\query_p(\dist)\geq\gamma\Delta n$.

We have already established that $\mech$ is $(0,\delta)$-max-KL stable.
Looking at Equation~\ref{eq:mem}, $\mech$ is trained with and without 
some element $q(\sample)$. The memorization score is computed as the
difference of post-processing when $\mech$ is trained with and without
the element. By post-processing (Lemma~\ref{lem:comp_gen}),
\textit{any} learned hypothesis $h$
will still differ in probabilities of classification on any input,
with and without training on $q(\sample)$,
by at most $\delta$. Thus,
$\mem(\mech, \sample, q(\sample)) \leq \delta$.

And we have also shown that
for $\sample\getsr\dist^n$ and query $\query\in Q_\Delta$,
$$
\Pr[\norm{\query(\sample) - \query(\dist)} \geq \gamma\Delta n]\geq\frac{\delta}{2\gamma},
$$
where $\dist$ is a distribution over $\univ$.
\end{proof}

We note that Max-KL stability is related to other notions of stability in the
literature, including KL-stability and TV-stability:

\begin{definition}[TV-Stability]
\label{def:TVstable}
Let $\mech \from \univ^n \to \calR$ be a randomized algorithm.  We say that $\mech$ is \emph{$\eps$-TV stable} if for every pair of samples that differ on exactly one element,
$$
\tvd(\mech(\sample), \mech(\sample')) = \sup_{R \subseteq \calR} \left| \pr{}{\mech(\sample) \in R} - \pr{}{\mech(\sample') \in R} \right| \leq \eps.
$$
\end{definition}

\begin{definition}[KL-Stability]
\label{def:KLstable}
Let $\mech \from \univ^n \to \calR$ be a randomized algorithm.  We say that $\mech$ is \emph{$\eps$-KL-stable} if for every pair of samples $\sample, \sample'$ that differ on exactly one element, 
$$
\ex{r \getsr \mech(\sample)}{\log\left( \frac{\pr{}{\mech(\sample) = r}}{\pr{}{\mech(\sample') = r}} \right)} \leq 2\eps^2
$$
\end{definition}

It can be shown that  $(\eps, \delta)$-max-KL stability implies $(2\eps + \delta)$-TV stability for $\eps\leq 1$ and $\delta > 0$~\citep{DworkRV10}. Because of
the relations between the stability notions, we have focused on one
such notion, the max-KL stability, and its properties.
\newpage
\section{Attack Details}
\label{sec:attack-details}

For each attack, we only change the input image $x$, while retaining its original label $y$, to retain class distributions in each dataset.

% \daniel{for each algorithm, you should explain what it does in at least a paragraph. e.g., Algorithm xxx implements the DeepFool attack of xxxx et al. $\hat{k}$ is a function that returns the class that example $x$ belongs to. }

\subsection{OOD Replacement Attack}

We first consider the robustness of data valuation to samples out-of-distribution of the original dataset. We use other known datasets to sample out-of-distribution images from. For MNIST, we attack using replacement samples from FashionMNIST. We attack SVHN with CIFAR-10 and vice-versa. The attack procedure is described in Algorithm~\ref{alg-ood}.

\begin{algorithm}[H]
\caption{OOD Replacement Attack}
    \begin{algorithmic}[1]
    \STATE \textbf{Input:} Image $x$ in dataset $S$, out-of-distribution dataset $S'$  
    \STATE \textbf{Output:} Image $\hat{x}$ sampled from $S'$
    \STATE
    \STATE Sample $x' \sim S'$
    \RETURN $\hat{x}\leftarrow x'$
    \end{algorithmic}
    \label{alg-ood}
\end{algorithm}

\subsection{Pseudoinverse Attack}

We treat each image as a matrix and take its pseudoinverse as the attack sample. To prevent vanishing gradient from pixel values being too close to 0 after the pseudoinverse operation, we multiply our perturbed images by a scalar factor to return images that can be processed by our neural networks. In PyTorch, images are converted to tensors in $[0,1]$, so our attack correspondingly scales the pseudoinverses to have unit normalization, described in Algorithm~\ref{alg-pinv}. For the RGB images of SVHN and CIFAR-10, we take the pseudoinverse of each channel and return the stacked pseudoinverses as our attack image. $x^+$ denotes the usual pseudoinverse operation in linear algebra.

% \varun{what is $x^+$ in the algorithm?}

\begin{algorithm}[H]
\caption{Pseudoinverse Perturbation Attack}
    \begin{algorithmic}[1]
    \STATE \textbf{Input:} Image $x$
    \STATE \textbf{Output:} Pseudoinverse $\hat{x}$ scaled to avoid vanishing gradient
    \STATE
    \RETURN $\hat{x}\leftarrow \frac{x^+}{||x^+||_1}$
    \end{algorithmic}
    \label{alg-pinv}
\end{algorithm}

\subsection{Naive Wasserstein Attack (Earth Mover's Distance)}

We aim to produce images that have maximal Wasserstein distance from the original image as possible. To reduce computational complexity, we implement a naive greedy binary search method (over the possible range of RGB values for images) that searches per pixel for the image with maximal Wasserstein distance, described in Algorithm~\ref{alg-naiveemd}. Note that $\mu_W(x,y)$ returns the Wasserstein distance between two images $x$ and $y$.

% \varun{what does $\hat{x}_{p\gets l}$ denote?}
$\hat{x}_{p\gets l}$ denotes the images $\hat{x}$ with pixel $p$ having value $l$,and $\hat{x}_{p\gets r}$ is similarly defined.

\begin{algorithm}[H]
\caption{Naive EMD Attack}
    \begin{algorithmic}[1]
    \STATE \textbf{Input:} Image $x$
    \STATE \textbf{Output:} Image $\hat{x}$ with high Wasserstein distance from original image
    \STATE
    \STATE Initialize $\hat{x}\leftarrow \mathbf{0}$
    \FOR{pixel $p$ in $x$}
        \STATE $l\leftarrow 0, r\leftarrow 255$
        \FOR{$i=1$ to 8}
            \IF{$\mu_W(\hat{x}_{p\gets l}, x) < \mu_W(\hat{x}_{p\gets r}, x)$}
                \STATE $l \leftarrow \frac{l+r}{2}$
            \ELSE
                \STATE $r \leftarrow \frac{l+r}{2}$
            \ENDIF
        \ENDFOR
        \STATE $\hat{x}\gets \hat{x}_{p\leftarrow l}$
    \ENDFOR
    \RETURN $\hat{x}$
    \end{algorithmic}
    \label{alg-naiveemd}
\end{algorithm}

\subsection{DeepFool Perturbation Attack}

% \varun{that's not his name? moosavi-dezfooli...}
\citet{moosavi2016deepfool} give an algorithm for DeepFool, an adversarial classification attack which perturbs an input image towards the decision boundary of a model to force misclassification. We attack data points by pushing them towards the decision boundary. The attack procedure is described in Algorithm~\ref{alg-deepfool2}, with \citet{moosavi2016deepfool}'s DeepFool subroutine described in Algorithm~\ref{alg-deepfool1}. We use overshoot parameter $\alpha=0.02$ as in the original DeepFool paper.

\begin{algorithm}[H]
\caption{DeepFool (Moosavi-Dezfooli et al.)}
    \begin{algorithmic}[1]
    \STATE \textbf{Input:} Image $x$, classifier $h$ which outputs class probabilities from an image, and $\hat{h}$ which outputs class labels from an image based on $h$
    \STATE \textbf{Output:} Perturbation $\hat{r}$ that pushes $x$ out of decision region.
    \STATE
    \STATE Initialize $x_0\leftarrow x, i\leftarrow 0$.
    \WHILE{$\hat{h}(x_i)=\hat{h}(x_0)$}
        \FOR{classes $k\neq \hat{h}(x_0)$}
            \STATE $\mathbf{w}'_k\leftarrow\nabla h_k(x_i)-\nabla h_{\hat{h}(x_0)}(x_i)$
            \STATE $h'_k\leftarrow h_k(x_i)-h_{\hat{h}(x_0)}(x_i)$
        \ENDFOR
    \STATE $\hat{l}\leftarrow\arg\min_{k\neq \hat{h}(x_0)}\frac{|h_k'|}{||\mathbf{w}'_k||_2}$
    \STATE $r_i\leftarrow\frac{|h'_{\hat{l}}|}{||\mathbf{w}'_{\hat{l}}||_2^2}\mathbf{w}_{\hat{l}}$
    \STATE $x_{i+1}\leftarrow x_i+r_i$
    \STATE $i\leftarrow i+1$
    \ENDWHILE
    \RETURN $\hat{r}=\sum_ir_i$
    \end{algorithmic}
    \label{alg-deepfool1}
\end{algorithm}

% \varun{the input to deepfool includes $\hat{h}$? that's missing in your algo} \tue{it's defined in the "Input"}
\begin{algorithm}[H]
\caption{Adversarial DeepFool Perturbation Attack}
    \begin{algorithmic}[1]
    \STATE \textbf{Input:} Image $x$ in dataset $S$, with training algorithm $\mathcal{A}$, overshoot $\alpha$
    \STATE \textbf{Output:} Image $\hat{x}$ such that model trained on $S$ with mislabel $y$ with minimal distance
    \STATE
    \STATE Train classifier $h\rightarrow\mathcal{A}(S)$
    \STATE $\hat{r}\leftarrow \text{DeepFool}(h,x)$.
    \RETURN $\hat{x}\leftarrow x+(1+\alpha)\hat{r}$
    \end{algorithmic}
    \label{alg-deepfool2}
\end{algorithm}

% \varun{respond to the questions below!}

% \begin{itemize}
% \item Datasets considered?
% \item Models considered?
% \item notions of memorization/influence considered? what approximations did you make?
% \item attacks considered?
% \item how many trials for each run?
% \item on what hardware the experiments were done?
% \end{itemize}

\newpage
\section{Additional Experimental Details}
\label{sec:additional-experiments}

%\subsection{Detailed Experimental Setup}

The codebase for our experiments is presented at \url{https://github.com/tuedo2/MemAttack}.

\subsection{Training Setup}

Over all models we use the SGD optimizer with learning rate of $10^{-2}$ and momentum 0.9. We use a batch size of 512 to speed up training. For models trained on MNIST we train for 5 epochs, and for the other two datasets SVHN and CIFAR10 we train for 10 epochs. All models are trained using the CrossEntropy loss function. All experiments were conducted using NVIDIA A40 single GPU nodes, and all training and attack procedures were implemented in PyTorch. Our architectures and training algorithms are not state-of-the-art since state-of-the-art training is significantly more
computationally intensive.

\subsection{Proxy Implementation Details}

\subsubsection{Loss Curvature}

For the Loss Curvature calculation, we take the curvature after epoch 2 for MNIST and after epoch 5 for SVHN and CIFAR10. We follow the code from \citet{ravikumar2024unveiling} which uses Hutchinson's trace estimator to approximate the Hessian product of randomly sampled Radamacher vectors for trace estimation, consistent with~\citet{garg2023memorization}.

\subsubsection{Learning Events}

We give a detailed description of the four learning events below, following~\citet{jiang2020characterizing,zhao2024scalability}. Each is implicitly evaluated for any data sample $(x,y)$.
\begin{enumerate}
    \item \textbf{Confidence:} The softmax probability of $h(x)$ corresponding to ground truth label $y$.
    \item \textbf{Max Confidence:} The highest softmax probability of $h(x)$ across all classes.
    \item \textbf{Entropy:} The entropy of the output probabilities of $h(x)$.
    \item \textbf{Binary Correctness:} Indicator of whether the model correctly predicts $y$ for $x$ (0 or 1).
\end{enumerate}
For each event, we evaluate each data point after each training data point, and then take the mean of the event across all epochs to return as our final learning event score.

\subsubsection{Privacy Risk Score}

We closely adhere to~\citet{song2021systematic} implementation for membership inference. We train our shadow models on the test split of each dataset, and then evaluate the membership of a target model trained over the training set (which is perturbed by our attacks) to calculate a score for each training data sample. The precise score comes from comparing histograms of of shadow and target models.

\subsection{Robustness of Attacks to Training Setting}

\noindent{\bf Significance of Dataset:} Figure~\ref{fig:proxy-resnet-all} compares the effectiveness of our four attacks over confidence event scoring on the MNIST data set across all model architectures we consider. This metric best showcases the advantage of \texttt{PINV}. Similar plots that demonstrate the resilience of \texttt{PINV} across underlying dataset are presented in Appendix~\ref{sec:additional-experiments}. We also note that the other attacks seem to exhibit similar trends across dataset, suggesting that memorization-based scoring is resilient for evaluating data for neural networks. Across all datasets, most attacks see a monotonic increase of scoring as the size of attack set increases. We observe that most attacks display their strongest performance over CIFAR10. Our plot suggests that more complex datasets are more susceptible to attack.
% \varun{what mem metric are you using and why?} 
% \varun{what do you want the reader to takeaway? note -- this takeaway should be soemthing new and not what is already present in 5.1} 
%\varun{is this takeaway going to be the same for other datasets and privacy metrics? if so, where can we find those graphs?}

\begin{figure}[h!]
    \centering
    \begin{subcaptionbox}{MNIST}[0.3\textwidth]
        {\includegraphics[width=\linewidth]{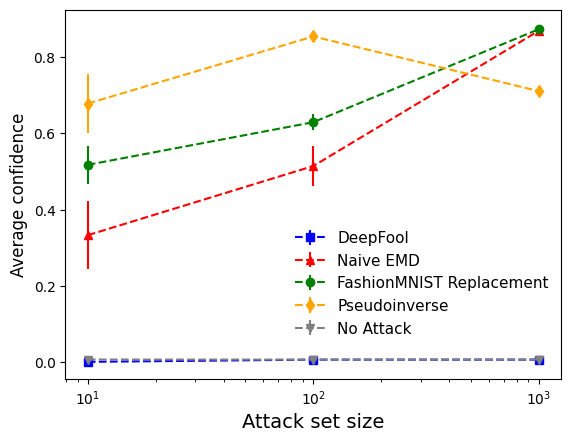}}
    \end{subcaptionbox}
    \hfill
    \begin{subcaptionbox}{SVHN}[0.3\textwidth]
        {\includegraphics[width=\linewidth]{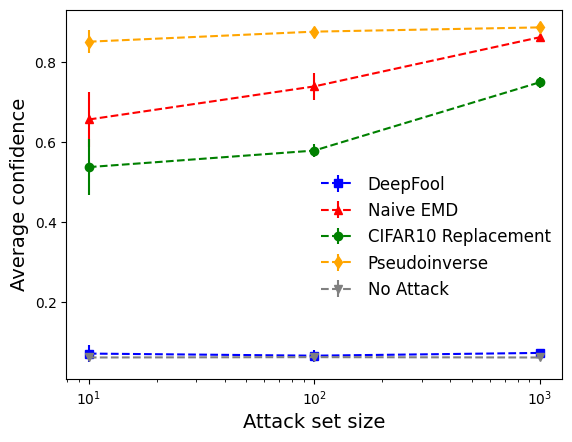}}
    \end{subcaptionbox}
    \hfill
    \begin{subcaptionbox}{CIFAR-10}[0.3\textwidth]
        {\includegraphics[width=\linewidth]{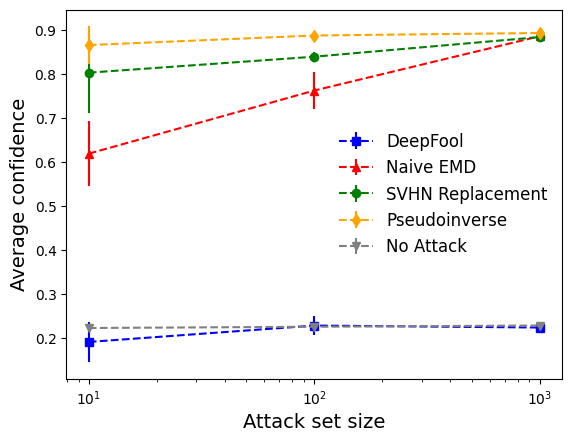}}
    \end{subcaptionbox}
    \caption{\textbf{Confidence event scoring across datasets.} \texttt{PINV} outperforms other attacks across datasets.}
    \label{fig:proxy-resnet-all}
\end{figure}

% \daniel{can you make the label and legend font sizes larger using matplotlib?}

\noindent{\bf Significance of Architecture:} 
% Tables~\ref{fig:vgg-100-table} and~\ref{fig:resnet-100-table} demonstrate aggregate metrics. To present results in more detail, 
Figure~\ref{fig:proxy-mnist-all} compares the effectiveness of our four attacks on just the ResNet-18 architecture across all benchmark datasets we consider using the over confidence event proxy metric.
% This metric best showcases the advantage of \texttt{PINV}. 
Similar plots that demonstrate the resilience of \texttt{PINV} across model architecture are presented in Appendix~\ref{sec:additional-experiments}. With the exception of \texttt{PINV}, we again observe a monotonic increase in attack performance with attack subset size across all architectures. On the VGG-11 models, \texttt{PINV} holds a weaker advantage compared to the other models. Although the relative ordering is preserved, we still observe some variance in scoring advantage across model architectures.

\begin{figure}[h!]
    \centering
    \begin{subcaptionbox}{VGG-11}[0.3\textwidth]
        {\includegraphics[width=\linewidth]{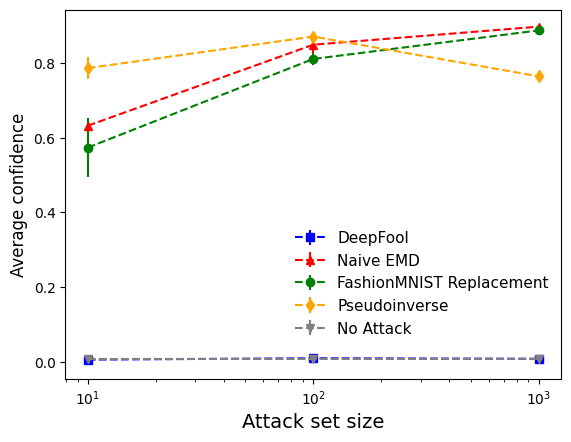}}
    \end{subcaptionbox}
    \hfill
    \begin{subcaptionbox}{ResNet-18}[0.3\textwidth]
        {\includegraphics[width=\linewidth]{assets/proxy/mnist_resnet.png}}
    \end{subcaptionbox}
    \hfill
    \begin{subcaptionbox}{MobileNet-v2}[0.3\textwidth]
        {\includegraphics[width=\linewidth]{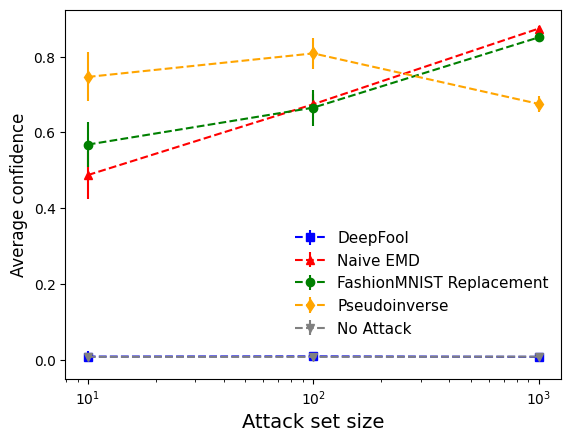}}
    \end{subcaptionbox}
    \caption{\textbf{Confidence event scoring across model architectures.} \texttt{PINV} outperforms other attacks across model architectures.}
    \label{fig:proxy-mnist-all}
\end{figure}

\subsection{Performance of Attacks Across Memorization Proxy}

Figure~\ref{fig:cifar-vgg-all} compares the effectiveness of our four attacks on a VGG-11 model trained on the CIFAR-10 dataset across all three memorization proxies we consider. Similar plots tracking \texttt{PINV} over different scoring metrics are found in Appendix~\ref{sec:additional-experiments}. Interestingly, curvature score does not exhibit the monotonic increase in performance observed in our other plots. We also find confidence event scoring to be the most susceptible to attack across all metrics considered. While we observe high variance in attack performance across metrics, the relative rankings of the attacks remain the same.

% \varun{what do you want the reader to takeaway? note -- this takeaway should be soemthing new and not what is already present in 5.1} 

\begin{figure}[h!]
    \centering
    \begin{subcaptionbox}{Curvature}[0.3\textwidth]
        {\includegraphics[width=\linewidth]{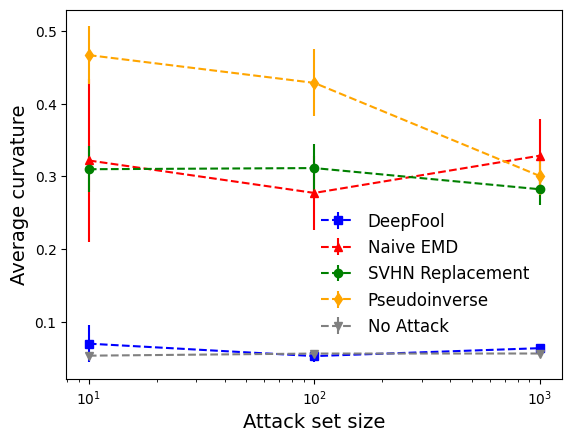}}
    \end{subcaptionbox}
    \hfill
    \begin{subcaptionbox}{Confidence}[0.3\textwidth]
        {\includegraphics[width=\linewidth]{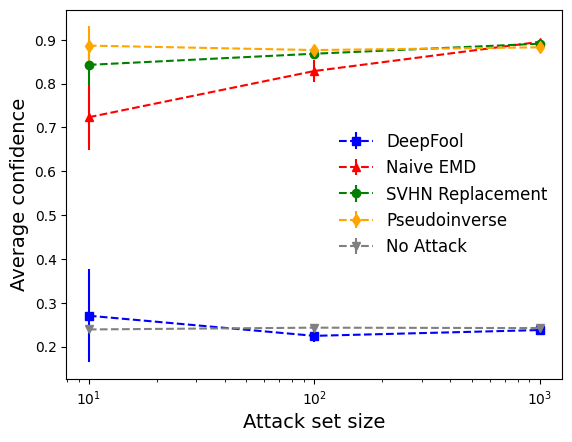}}
    \end{subcaptionbox}
    \hfill
    \begin{subcaptionbox}{Privacy Score}[0.3\textwidth]
        {\includegraphics[width=\linewidth]{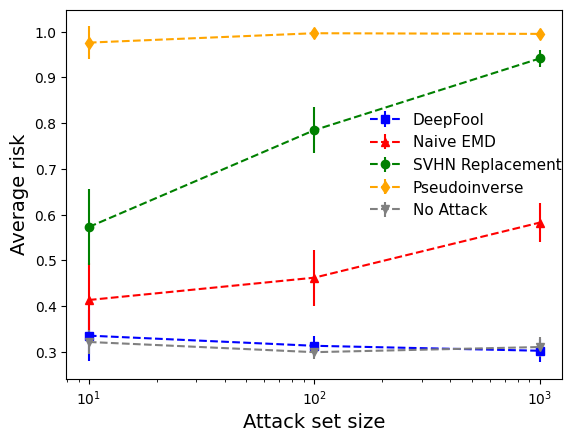}}
    \end{subcaptionbox}
    \caption{\textbf{Proxy scores on VGG-11 model architectures.} \texttt{PINV} outperforms other attacks across memorization proxies}
    \label{fig:cifar-vgg-all}
\end{figure}

\subsection{Additional Comparison Tables}

\subsubsection{VGG-11 Results}

Tables~\ref{fig:vgg-10-table},~\ref{fig:vgg-100-table},~\ref{fig:vgg-1000-table} present the mean and variance of our four attacks on the VGG-11 architecture across all of our experimental settings. As discussed in \S~\ref{sec:results}, \texttt{PINV} outperforms other attacks but loses its advantage when the attack size reaches $10^4$. Although \texttt{PINV} is observed to lose out on its top spot in certain settings, we still observe it to consistently produce perturbed points with high advantage over non-attack points. \texttt{PINV} performs the best in the SVHN dataset under the VGG-11 architecture.

\begin{table}[H]
\centering
\resizebox{\textwidth}{!}{%
\begin{tabular}{c|ccc|ccc|ccc}
\toprule
\textbf{Attack} & \multicolumn{3}{c|}{\textbf{Loss Curvature}} & \multicolumn{3}{c|}{\textbf{Confidence Event}} & \multicolumn{3}{c}{\textbf{Privacy Score}} \\
\cmidrule(r){2-4} \cmidrule(r){5-7} \cmidrule(r){8-10}
& MNIST & SVHN & CIFAR-10 & MNIST & SVHN & CIFAR-10 & MNIST & SVHN & CIFAR-10 \\
\midrule
\texttt{None} & 0.00\pp0.00 & 0.01\pp0.00 & 0.05\pp0.00 & 0.01\pp0.00 & 0.07\pp0.00 & 0.24\pp0.00 & 0.50\pp0.00 & 0.49\pp0.00 & 0.32\pp0.00 \\
\midrule
\texttt{OOD} & 0.11\pp0.00 & 0.01\pp0.00 & 0.25\pp0.00 & 0.57\pp0.01 & 0.52\pp0.01 & 0.61\pp0.00 & 0.20\pp0.00 & 0.05\pp0.01 & 0.26\pp0.00  \\
\texttt{PINV} & 0.11\pp0.00 & \textbf{0.17\pp0.00} & \textbf{0.41\pp0.00} & \textbf{0.78\pp0.01} & \textbf{0.77\pp0.00} & \textbf{0.65\pp0.00} & \textbf{0.33\pp0.01} & \textbf{0.51\pp0.00} & \textbf{0.65\pp0.00}  \\
\texttt{EMD} & \textbf{0.17\pp0.00} & 0.04\pp0.00 & 0.27\pp0.00 & 0.62\pp0.00 & 0.65\pp0.00 & 0.48\pp0.01 & 0.06\pp0.01 & 0.06\pp0.00 & 0.11\pp0.01 \\
\texttt{DF} & 0.00\pp0.00 & 0.00\pp0.00 & 0.01\pp0.00 & 0.00\pp0.00 & 0.04\pp0.00 & 0.03\pp0.01 & 0.01\pp0.00 & 0.03\pp0.00 & 0.03\pp0.00 \\
\bottomrule
\end{tabular}%
}
\caption{\textbf{EAA on VGG-11 architecture with attack set size of 10.} \texttt{PINV} outperforms almost all other attacks by a significant margin across dataset and proxy.}
\label{fig:vgg-10-table}
\end{table}

\begin{table}[H]
\centering
\resizebox{\textwidth}{!}{%
\begin{tabular}{c|ccc|ccc|ccc}
\toprule
\textbf{Attack} & \multicolumn{3}{c|}{\textbf{Loss Curvature}} & \multicolumn{3}{c|}{\textbf{Confidence Event}} & \multicolumn{3}{c}{\textbf{Privacy Score}} \\
\cmidrule(r){2-4} \cmidrule(r){5-7} \cmidrule(r){8-10}
& MNIST & SVHN & CIFAR-10 & MNIST & SVHN & CIFAR-10 & MNIST & SVHN & CIFAR-10 \\
\midrule
\texttt{None} & 0.01\pp0.00 & 0.01\pp0.00 & 0.06\pp0.00 & 0.01\pp0.00 & 0.07\pp0.00 & 0.24\pp0.00 & 0.50\pp0.00 & 0.49\pp0.00 & 0.30\pp0.00 \\
\midrule
\texttt{OOD} & 0.07\pp0.00 & 0.01\pp0.00 & 0.25\pp0.00 & 0.80\pp0.00 & 0.61\pp0.00 & 0.63\pp0.00 & 0.36\pp0.00 & 0.06\pp0.00 & 0.49\pp0.00  \\
\texttt{PINV} & 0.08\pp0.00 & \textbf{0.23\pp0.00} & \textbf{0.37\pp0.00} & \textbf{0.85\pp0.00} & \textbf{0.81\pp0.00} & \textbf{0.66\pp0.00} & \textbf{0.43\pp0.00} & \textbf{0.51\pp0.00} & \textbf{0.70\pp0.00}  \\
\texttt{EMD} & \textbf{0.12\pp0.00} & 0.04\pp0.00 & 0.22\pp0.00 & 0.84\pp0.00 & 0.74\pp0.00 & 0.59\pp0.00 & 0.31\pp0.00 & 0.08\pp0.00 & 0.15\pp0.00 \\
\texttt{DF} & 0.00\pp0.00 & 0.00\pp0.00 & 0.00\pp0.00 & 0.00\pp0.00 & 0.00\pp0.00 & 0.00\pp0.00 & 0.01\pp0.00 & 0.01\pp0.00 & 0.01\pp0.00 \\
\bottomrule
\end{tabular}%
}
\caption{\textbf{EAA on VGG-11 architecture with attack set size of 100.} \texttt{PINV} outperforms almost all other attacks by a significant margin across dataset and proxy.}
\label{fig:vgg-100-table}
\end{table}

\begin{table}[H]
\centering
\resizebox{\textwidth}{!}{%
\begin{tabular}{c|ccc|ccc|ccc}
\toprule
\textbf{Attack} & \multicolumn{3}{c|}{\textbf{Loss Curvature}} & \multicolumn{3}{c|}{\textbf{Confidence Event}} & \multicolumn{3}{c}{\textbf{Privacy Score}} \\
\cmidrule(r){2-4} \cmidrule(r){5-7} \cmidrule(r){8-10}
& MNIST & SVHN & CIFAR-10 & MNIST & SVHN & CIFAR-10 & MNIST & SVHN & CIFAR-10 \\
\midrule
\texttt{None} & 0.00\pp0.00 & 0.01\pp0.00 & 0.06\pp0.00 & 0.01\pp0.00 & 0.07\pp0.00 & 0.24\pp0.00 & 0.51\pp0.00 & 9,49\pp0.00 & 0.31\pp0.00 \\
\midrule
\texttt{OOD} & 0.03\pp0.00 & 0.01\pp0.00 & 0.21\pp0.00 & 0.88\pp0.00 & 0.75\pp0.00 & \textbf{0.65\pp0.00} & \textbf{0.50\pp0.00} & 0.15\pp0.00 & 0.64\pp0.00  \\
\texttt{PINV} & \textbf{0.06\pp0.00} & \textbf{0.09\pp0.00} & 0.24\pp0.00 & 0.76\pp0.00 & \textbf{0.81\pp0.00} & 0.64\pp0.00 & 0.42\pp0.00 & \textbf{0.51\pp0.00} & \textbf{0.68\pp0.00}  \\
\texttt{EMD} & \textbf{0.06\pp0.00} & 0.07\pp0.00 & \textbf{0.28\pp0.00} & \textbf{0.89\pp0.00} & 0.80\pp0.00 & \textbf{0.65\pp0.00} & \textbf{0.50\pp0.00} & 0.15\pp0.01 & 0.27\pp0.00 \\
\texttt{DF} & 0.00\pp0.00 & 0.00\pp0.00 & 0.00\pp0.00 & 0.00\pp0.00 & 0.01\pp0.00 & 0.00\pp0.00 & 0.00\pp0.00 & 0.00\pp0.00 & 0.00\pp0.00 \\
\bottomrule
\end{tabular}%
}
\caption{\textbf{EAA on VGG-11 architecture with attack set size of 1000.} \texttt{OOD}, \texttt{PINV}, \texttt{EMD} all perform similarly across dataset and proxy.}
\label{fig:vgg-1000-table}
\end{table}

\subsubsection{ResNet-18 Results}

Tables~\ref{fig:resnet-10-table},~\ref{fig:resnet-100-table},~\ref{fig:resnet-1000-table} present the mean and variance of our four attacks on the ResNet-18 architecture across all of our experimental settings. We again observe that \texttt{PINV} outperforms other attacks but slightly loses its advantage when the attack size reaches $10^4$. Compared to the results over the VGG-11 architecture, it appears that \texttt{PINV} displays more consistent top performances over the ResNet-18 model. We again observe that \texttt{PINV} performs relatively the best in the SVHN dataset.

\begin{table}[H]
\centering
\resizebox{\textwidth}{!}{%
\begin{tabular}{c|ccc|ccc|ccc}
\toprule
\textbf{Attack} & \multicolumn{3}{c|}{\textbf{Loss Curvature}} & \multicolumn{3}{c|}{\textbf{Confidence Event}} & \multicolumn{3}{c}{\textbf{Privacy Score}} \\
\cmidrule(r){2-4} \cmidrule(r){5-7} \cmidrule(r){8-10}
& MNIST & SVHN & CIFAR-10 & MNIST & SVHN & CIFAR-10 & MNIST & SVHN & CIFAR-10 \\
\midrule
\texttt{None} & 0.01\pp0.00 & 0.01\pp0.00 & 0.09\pp0.00 & 0.01\pp0.00 & 0.06\pp0.00 & 0.23\pp0.00 & 0.53\pp0.00 & 0.49\pp0.00 & 0.21\pp0.00 \\
\midrule
\texttt{OOD} & 0.07\pp0.00 & 0.02\pp0.00 & \textbf{0.04\pp0.00} & \textbf{0.86\pp0.00} & 0.68\pp0.00 & 0.65\pp0.00 & 0.51\pp0.01 & \textbf{0.49\pp0.00} & \textbf{0.23\pp0.00}  \\
\texttt{PINV} & \textbf{0.10\pp0.00} & \textbf{0.03\pp0.00} & 0.02\pp0.00 & 0.70\pp0.00 & \textbf{0.82\pp0.00} & \textbf{0.66\pp0.00} & \textbf{0.53\pp0.00} & \textbf{0.49\pp0.00} & 0.21\pp0.00  \\
\texttt{EMD} & 0.03\pp0.00 & 0.00\pp0.00 & -0.04\pp0.00 & \textbf{0.86\pp0.00} & 0.80\pp0.00 & \textbf{0.66\pp0.00} & 0.52\pp0.00 & \textbf{0.49\pp0.00} & 0.18\pp0.00 \\
\texttt{DF} & 0.00\pp0.00 & 0.00\pp0.00 & 0.00\pp0.00 & 0.00\pp0.00 & 0.01\pp0.00 & 0.00\pp0.00 & 0.52\pp0.00 & 0.48\pp0.00 & 0.21\pp0.00 \\
\bottomrule
\end{tabular}%
}
\caption{\textbf{EAA on ResNet-18 architecture with attack set size of 1000.} \texttt{OOD}, \texttt{PINV}, \texttt{EMD} all perform similarly across dataset and proxy.}
\label{fig:resnet-1000-table}
\end{table}

\subsubsection{MobileNet-v2 Results}

Tables~\ref{fig:mobile-10-table},~\ref{fig:mobile-100-table},~\ref{fig:mobile-1000-table} present the mean and variance of our four attacks on the MobileNetV2 architecture across all of our experimental settings. We find that over the MobileNetV2 architecture, \texttt{PINV} displays the most resilience with the Loss Curvature proxy. Furthermore, \texttt{PINV} remains a strong performer across the other proxies.

\begin{table}[H]
\centering
\resizebox{\textwidth}{!}{%
\begin{tabular}{c|ccc|ccc|ccc}
\toprule
\textbf{Attack} & \multicolumn{3}{c|}{\textbf{Loss Curvature}} & \multicolumn{3}{c|}{\textbf{Confidence Event}} & \multicolumn{3}{c}{\textbf{Privacy Score}} \\
\cmidrule(r){2-4} \cmidrule(r){5-7} \cmidrule(r){8-10}
& MNIST & SVHN & CIFAR-10 & MNIST & SVHN & CIFAR-10 & MNIST & SVHN & CIFAR-10 \\
\midrule
\texttt{None} & 0.00\pp0.00 & 0.01\pp0.00 & 0.12\pp0.00 & 0.01\pp0.00 & 0.09\pp0.00 & 0.33\pp0.00 & 0.45\pp0.02 & 0.49\pp0.00 & 0.47\pp0.00 \\
\midrule
\texttt{OOD} & 0.10\pp0.00 & 0.03\pp0.01 & 0.27\pp0.00 & 0.56\pp0.00 & 0.65\pp0.00 & \textbf{0.54\pp0.00} & \textbf{0.59\pp0.00} & \textbf{0.50\pp0.00} & 0.41\pp0.00  \\
\texttt{PINV} & \textbf{0.22\pp0.00} & \textbf{0.19\pp0.00} & \textbf{0.46\pp0.00} & \textbf{0.74\pp0.00} & \textbf{0.77\pp0.00} & 0.51\pp0.00 & 0.45\pp0.02 & 0.49\pp0.00 & \textbf{0.47\pp0.00}  \\
\texttt{EMD} & 0.04\pp0.00 & 0.01\pp0.00 & -0.04\pp0.00 & 0.48\pp0.00 & 0.73\pp0.00 & 0.51\pp0.00 & 0.51\pp0.02 & 0.49\pp0.00 & 0.40\pp0.01 \\
\texttt{DF} & 0.00\pp0.00 & 0.00\pp0.00 & -0.03\pp0.00 & 0.00\pp0.00 & 0.00\pp0.00 & -0.06\pp0.00 & 0.50\pp0.00 & 0.49\pp0.00 & \textbf{0.47\pp0.02} \\
\bottomrule
\end{tabular}%
}
\caption{\textbf{EAA on MobileNet-v2 architecture with attack set size of 10.} \texttt{PINV} outperforms almost all other attacks by a slight margin across dataset and proxy.}
\label{fig:mobile-10-table}
\end{table}

\begin{table}[H]
\centering
\resizebox{\textwidth}{!}{%
\begin{tabular}{c|ccc|ccc|ccc}
\toprule
\textbf{Attack} & \multicolumn{3}{c|}{\textbf{Loss Curvature}} & \multicolumn{3}{c|}{\textbf{Confidence Event}} & \multicolumn{3}{c}{\textbf{Privacy Score}} \\
\cmidrule(r){2-4} \cmidrule(r){5-7} \cmidrule(r){8-10}
& MNIST & SVHN & CIFAR-10 & MNIST & SVHN & CIFAR-10 & MNIST & SVHN & CIFAR-10 \\
\midrule
\texttt{None} & 0.00\pp0.00 & 0.01\pp0.00 & 0.11\pp0.00 & 0.01\pp0.00 & 0.09\pp0.00 & 0.33\pp0.00 & 0.46\pp0.01 & 0.48\pp0.00 & 0.40\pp0.01 \\
\midrule
\texttt{OOD} & 0.10\pp0.00 & 0.02\pp0.00 & 0.22\pp0.00 & 0.66\pp0.00 & 0.62\pp0.00 & 0.53\pp0.00 & 0.53\pp0.02 & \textbf{0.50\pp0.00} & \textbf{0.43\pp0.02}  \\
\texttt{PINV} & \textbf{0.15\pp0.00} & \textbf{0.16\pp0.00} & \textbf{0.35\pp0.00} & \textbf{0.80\pp0.00} & \textbf{0.77\pp0.00} & \textbf{0.55\pp0.00} & 0.46\pp0.01 & 0.48\pp0.00 & 0.40\pp0.00  \\
\texttt{EMD} & 0.03\pp0.00 & 0.01\pp0.00 & -0.04\pp0.00 & 0.67\pp0.00 & \textbf{0.77\pp0.00} & 0.54\pp0.00 & \textbf{0.61\pp0.00} & \textbf{0.50\pp0.00} & \textbf{0.43\pp0.01} \\
\texttt{DF} & 0.00\pp0.00 & 0.00\pp0.00 & 0.01\pp0.00 & 0.00\pp0.00 & 0.00\pp0.00 & 0.02\pp0.00 & 0.60\pp0.01 & \textbf{0.50\pp0.00} & 0.37\pp0.00 \\
\bottomrule
\end{tabular}%
}
\caption{\textbf{EAA on MobileNet-v2 architecture with attack set size of 100.} \texttt{OOD}, \texttt{PINV}, \texttt{EMD} all perform similarly across dataset and proxy.}
\label{fig:mobile-100-table}
\end{table}

\begin{table}[H]
\centering
\resizebox{\textwidth}{!}{%
\begin{tabular}{c|ccc|ccc|ccc}
\toprule
\textbf{Attack} & \multicolumn{3}{c|}{\textbf{Loss Curvature}} & \multicolumn{3}{c|}{\textbf{Confidence Event}} & \multicolumn{3}{c}{\textbf{Privacy Score}} \\
\cmidrule(r){2-4} \cmidrule(r){5-7} \cmidrule(r){8-10}
& MNIST & SVHN & CIFAR-10 & MNIST & SVHN & CIFAR-10 & MNIST & SVHN & CIFAR-10 \\
\midrule
\texttt{None} & 0.01\pp0.00 & 0.01\pp0.00 & 0.12\pp0.00 & 0.01\pp0.00 & 0.09\pp0.00 & 0.33\pp0.00 & 0.40\pp0.01 & 0.52\pp0.00 & 0.41\pp0.02 \\
\midrule
\texttt{OOD} & 0.08\pp0.00 & 0.02\pp0.00 & 0.11\pp0.00 & 0.84\pp0.00 & 0.72\pp0.00 & \textbf{0.55\pp0.00} & 0.40\pp0.01 & 0.50\pp0.00 & 0.34\pp0.01  \\
\texttt{PINV} & \textbf{0.14\pp0.00} & \textbf{0.07\pp0.00} & \textbf{0.28\pp0.01} & 0.67\pp0.00 & \textbf{0.79\pp0.00} & 0.54\pp0.00 & 0.40\pp0.01 & \textbf{0.52\pp0.00} & 0.41\pp0.01  \\
\texttt{EMD} & 0.03\pp0.00 & 0.00\pp0.00 & -0.05\pp0.00 & \textbf{0.87\pp0.00} & \textbf{0.79\pp0.00} & \textbf{0.55\pp0.00} & \textbf{0.64\pp0.01} & 0.50\pp0.00 & 0.34\pp0.01 \\
\texttt{DF} & 0.00\pp0.00 & 0.00\pp0.00 & 0.00\pp0.00 & 0.00\pp0.00 & 0.01\pp0.00 & -0.01\pp0.00 & 0.44\pp0.00 & 0.49\pp0.00 & \textbf{0.44\pp0.01} \\
\bottomrule
\end{tabular}%
}
\caption{\textbf{EAA on MobileNet-v2 architecture with attack set size of 1000.} \texttt{OOD}, \texttt{PINV}, \texttt{EMD} all perform similarly across dataset and proxy.}
\label{fig:mobile-1000-table}
\end{table}

\subsection{Additional Comparison Plots}

\subsubsection{Curvature Proxy Plots}

Comparison plots of the four attacks for the Loss Curvature proxy across dataset and model architecture are presented in Figure~\ref{fig:3x3grid-curvature}. Unlike the other proxies, the absolute value of attack curvature decreases as the size of the attack set increases in \texttt{PINV}. The other attacks seem more resilient to this monotonic decrease, which follows our observations in \S~\ref{sec:results} where we assert that \texttt{PINV} loses its advantage as the attack size increases. With the exception of 7a., we see that \texttt{PINV} demonstrates superiority over the other attacks across all training settings over the Loss Curvature proxy.

\begin{figure}[H]
    \centering
    
    % First row
    \begin{subfigure}[b]{0.3\textwidth}
        \centering
        \includegraphics[width=\textwidth]{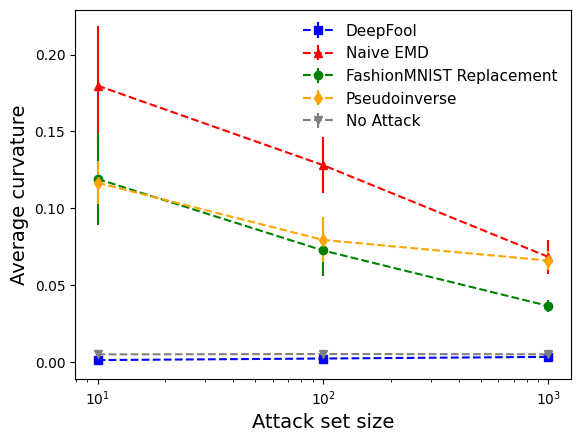}
        \caption{VGG-11 on MNIST}
    \end{subfigure}
    \hfill
    \begin{subfigure}[b]{0.3\textwidth}
        \centering
        \includegraphics[width=\textwidth]{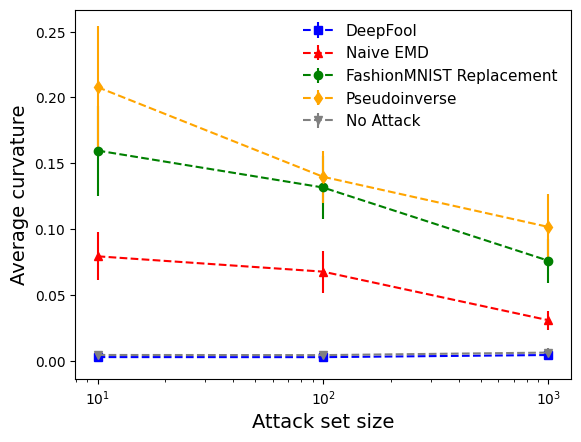}
        \caption{ResNet-18 on MNIST}
    \end{subfigure}
    \hfill
    \begin{subfigure}[b]{0.3\textwidth}
        \centering
        \includegraphics[width=\textwidth]{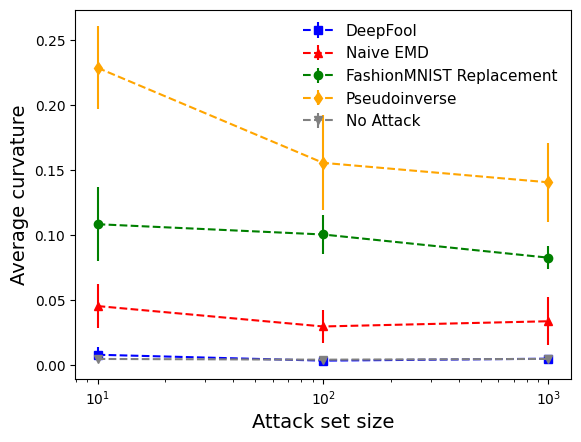}
        \caption{MobileNet-v2 on MNIST}
    \end{subfigure}
    
    \vspace{1em}
    
    % Second row
    \begin{subfigure}[b]{0.3\textwidth}
        \centering
        \includegraphics[width=\textwidth]{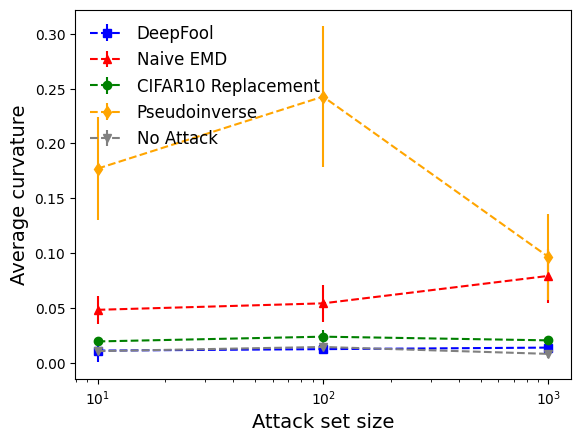}
        \caption{VGG-11 on SVHN}
    \end{subfigure}
    \hfill
    \begin{subfigure}[b]{0.3\textwidth}
        \centering
        \includegraphics[width=\textwidth]{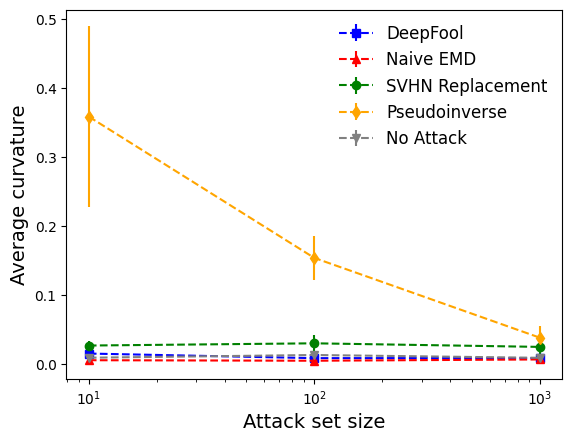}
        \caption{ResNet-18 on SVHN}
    \end{subfigure}
    \hfill
    \begin{subfigure}[b]{0.3\textwidth}
        \centering
        \includegraphics[width=\textwidth]{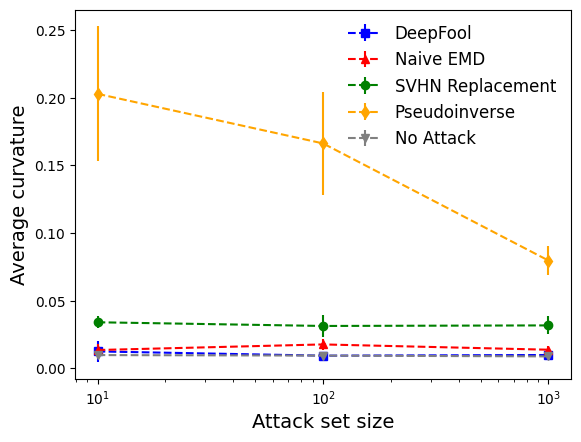}
        \caption{MobileNet-v2 on SVHN}
    \end{subfigure}
    
    \vspace{1em}
    
    % Third row
    \begin{subfigure}[b]{0.3\textwidth}
        \centering
        \includegraphics[width=\textwidth]{assets/curvature/cifar10_vgg.png}
        \caption{VGG-11 on CIFAR10}
    \end{subfigure}
    \hfill
    \begin{subfigure}[b]{0.3\textwidth}
        \centering
        \includegraphics[width=\textwidth]{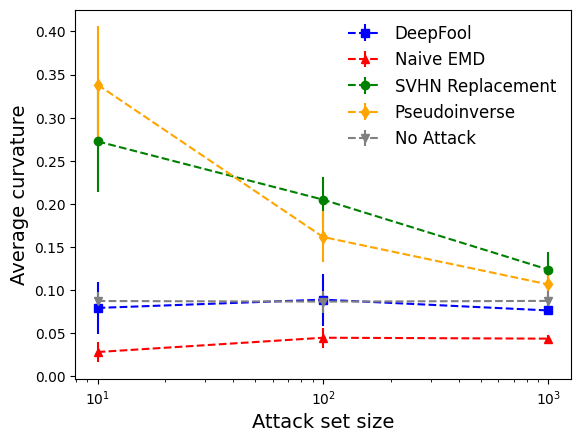}
        \caption{ResNet-18 on CIFAR10}
    \end{subfigure}
    \hfill
    \begin{subfigure}[b]{0.3\textwidth}
        \centering
        \includegraphics[width=\textwidth]{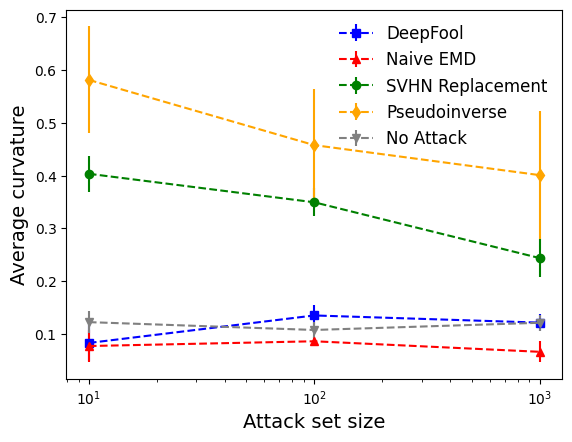}
        \caption{MobileNet-v2 on CIFAR10}
    \end{subfigure}

    \caption{\textbf{Attack comparison plots for Loss Curvature.} \texttt{PINV} outperforms other attacks across dataset and model architecture.}
    \label{fig:3x3grid-curvature}
\end{figure}

\newpage
\subsubsection{Confidence Proxy Plots}

Comparison plots of the four attacks for the Confidence Event proxy across dataset and model architecture are presented in Figure~\ref{fig:3x3grid-proxy}. With some exceptions at attack size $10^4$, we see that \texttt{PINV} strongly outperforms the other attacks in the Confidence Event proxy across all other experimental parameters. We also note that the performance of all four attacks seem consistent across attack set size, which differs from our observations in the other two proxies.

\begin{figure}[H]
    \centering
    
    % First row
    \begin{subfigure}[b]{0.3\textwidth}
        \centering
        \includegraphics[width=\textwidth]{assets/proxy/mnist_vgg.png}
        \caption{VGG-11 on MNIST}
    \end{subfigure}
    \hfill
    \begin{subfigure}[b]{0.3\textwidth}
        \centering
        \includegraphics[width=\textwidth]{assets/proxy/mnist_resnet.png}
        \caption{ResNet-18 on MNIST}
    \end{subfigure}
    \hfill
    \begin{subfigure}[b]{0.3\textwidth}
        \centering
        \includegraphics[width=\textwidth]{assets/proxy/mnist_mobile.png}
        \caption{MobileNet-v2 on MNIST}
    \end{subfigure}
    
    \vspace{1em}
    
    % Second row
    \begin{subfigure}[b]{0.3\textwidth}
        \centering
        \includegraphics[width=\textwidth]{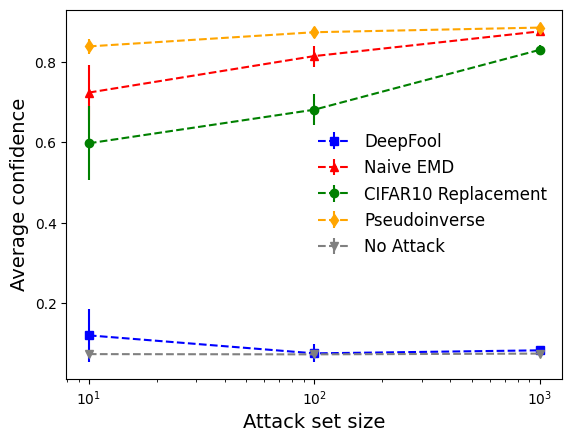}
        \caption{VGG-11 on SVHN}
    \end{subfigure}
    \hfill
    \begin{subfigure}[b]{0.3\textwidth}
        \centering
        \includegraphics[width=\textwidth]{assets/proxy/svhn_resnet.png}
        \caption{ResNet-18 on SVHN}
    \end{subfigure}
    \hfill
    \begin{subfigure}[b]{0.3\textwidth}
        \centering
        \includegraphics[width=\textwidth]{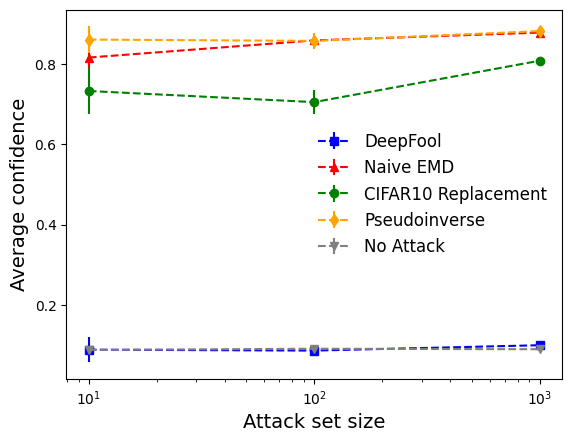}
        \caption{MobileNet-v2 on SVHN}
    \end{subfigure}
    
    \vspace{1em}
    
    % Third row
    \begin{subfigure}[b]{0.3\textwidth}
        \centering
        \includegraphics[width=\textwidth]{assets/proxy/cifar10_vgg.png}
        \caption{VGG-11 on CIFAR10}
    \end{subfigure}
    \hfill
    \begin{subfigure}[b]{0.3\textwidth}
        \centering
        \includegraphics[width=\textwidth]{assets/proxy/cifar10_resnet.png}
        \caption{ResNet-18 on CIFAR10}
    \end{subfigure}
    \hfill
    \begin{subfigure}[b]{0.3\textwidth}
        \centering
        \includegraphics[width=\textwidth]{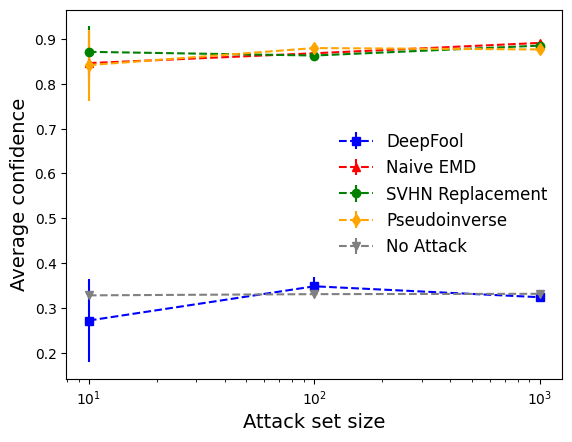}
        \caption{MobileNet-v2 on CIFAR10}
    \end{subfigure}

    \caption{\textbf{Attack comparison plots for Confidence Event.} \texttt{PINV} outperforms other attacks across dataset and model architecture.}
    \label{fig:3x3grid-proxy}
\end{figure}

\newpage
\subsubsection{Privacy Score Proxy Plots}

Comparison plots of the four attacks for the Risk Score proxy across dataset and model architecture are presented in Figure~\ref{fig:3x3grid-mia}. Here we see again that \texttt{PINV} is a consistently strong performer across all experimental parameters. Outside of MNIST, \texttt{PINV} solidly trumps all other attacks. We observe that while MobileNetV2 admits a wider range of performances for the other attacks in comparison to the other architectures, but none consistently scores higher than \texttt{PINV}.

\begin{figure}[H]
    \centering
    
    % First row
    \begin{subfigure}[b]{0.3\textwidth}
        \centering
        \includegraphics[width=\textwidth]{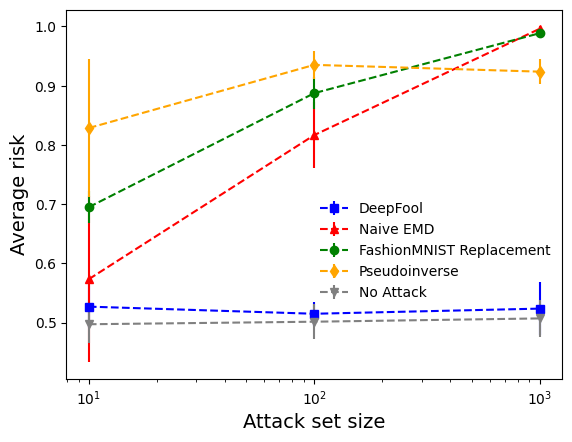}
        \caption{VGG-11 on MNIST}
    \end{subfigure}
    \hfill
    \begin{subfigure}[b]{0.3\textwidth}
        \centering
        \includegraphics[width=\textwidth]{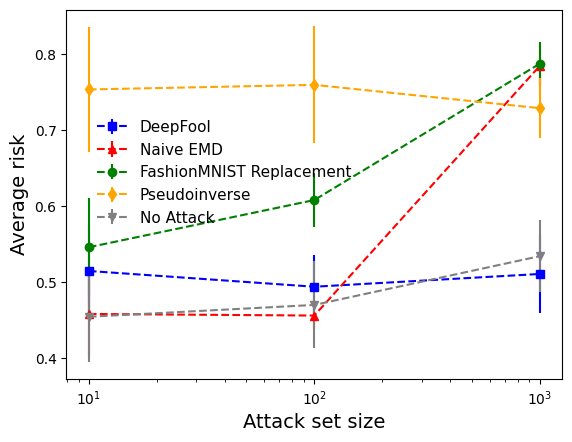}
        \caption{ResNet-18 on MNIST}
    \end{subfigure}
    \hfill
    \begin{subfigure}[b]{0.3\textwidth}
        \centering
        \includegraphics[width=\textwidth]{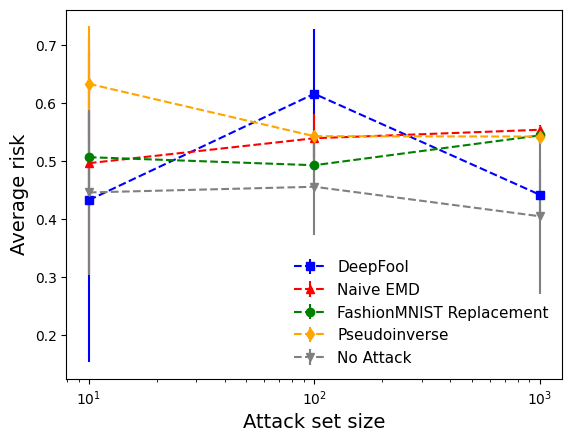}
        \caption{MobileNet-v2 on MNIST}
    \end{subfigure}
    
    \vspace{1em}
    
    % Second row
    \begin{subfigure}[b]{0.3\textwidth}
        \centering
        \includegraphics[width=\textwidth]{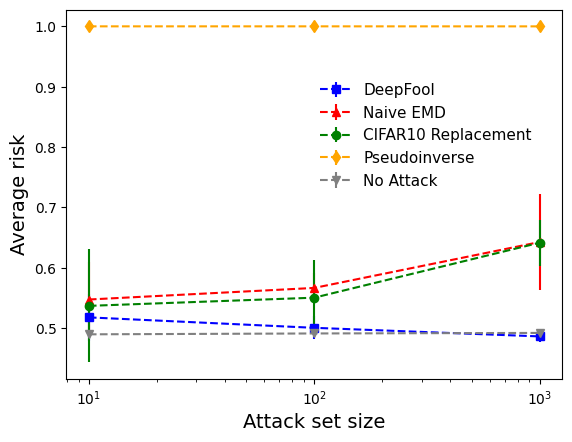}
        \caption{VGG-11 on SVHN}
    \end{subfigure}
    \hfill
    \begin{subfigure}[b]{0.3\textwidth}
        \centering
        \includegraphics[width=\textwidth]{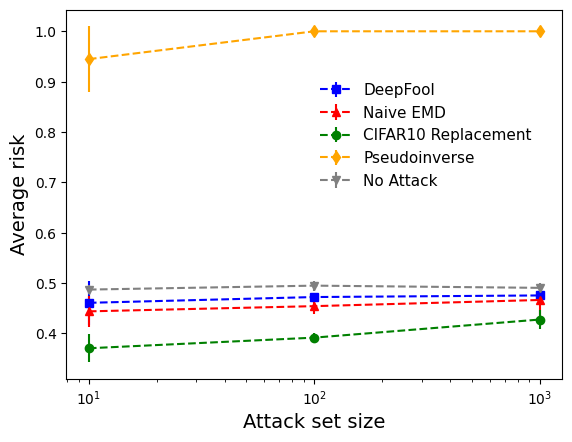}
        \caption{ResNet-18 on SVHN}
    \end{subfigure}
    \hfill
    \begin{subfigure}[b]{0.3\textwidth}
        \centering
        \includegraphics[width=\textwidth]{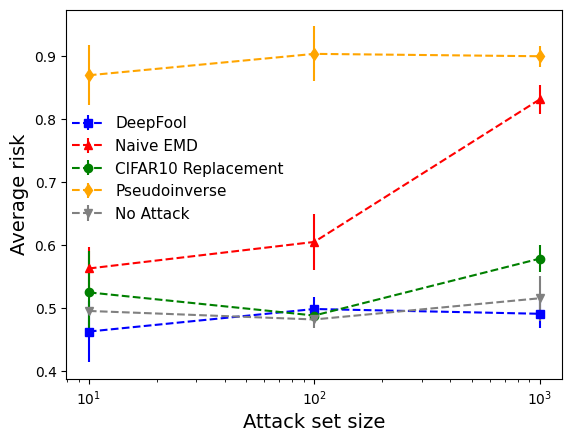}
        \caption{MobileNet-v2 on SVHN}
    \end{subfigure}
    
    \vspace{1em}
    
    % Third row
    \begin{subfigure}[b]{0.3\textwidth}
        \centering
        \includegraphics[width=\textwidth]{assets/mia/cifar10_vgg.png}
        \caption{VGG-11 on CIFAR10}
    \end{subfigure}
    \hfill
    \begin{subfigure}[b]{0.3\textwidth}
        \centering
        \includegraphics[width=\textwidth]{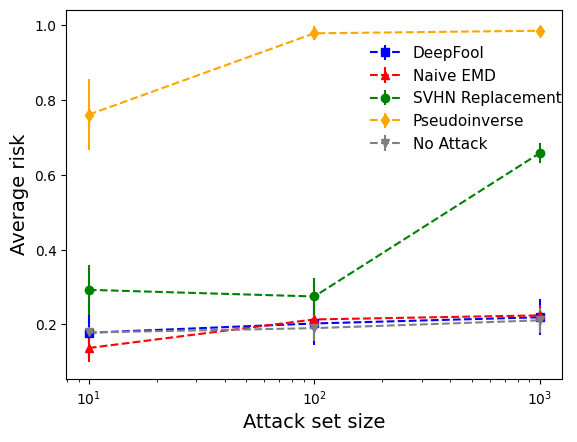}
        \caption{ResNet-18 on CIFAR10}
    \end{subfigure}
    \hfill
    \begin{subfigure}[b]{0.3\textwidth}
        \centering
        \includegraphics[width=\textwidth]{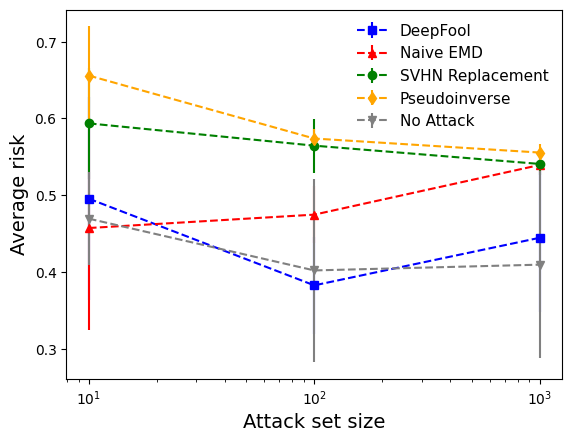}
        \caption{MobileNet-v2 on CIFAR10}
    \end{subfigure}

    \caption{\textbf{Attack comparison plots for Privacy Risk.} \texttt{PINV} outperforms other attacks across dataset and model architecture.}
    \label{fig:3x3grid-mia}
\end{figure}

\clearpage
\subsection{Other Learning Events}

In accordance with \citet{zhao2024scalability}, we consider 3 learning event proxies other than the confidence event detailed in the main paper. In line with the main results, we observe that \texttt{PINV} outperforms the other attacks regardless of the learning event considered. We also find that the relative ordering of the other attacks is preserved, which affirms the results of~\citet{jiang2020characterizing,zhao2024scalability} which find that these learning events have high correlation with label memorization, and each other.

\begin{table}[h!]
\centering
\resizebox{\textwidth}{!}{%
\begin{tabular}{c|ccc|ccc|ccc}
\toprule
\textbf{Attack} & \multicolumn{3}{c|}{\textbf{Max Confidence Event}} & \multicolumn{3}{c|}{\textbf{Entropy Event}} & \multicolumn{3}{c}{\textbf{Binary Correctness Event}} \\
\cmidrule(r){2-4} \cmidrule(r){5-7} \cmidrule(r){8-10}
& MNIST & SVHN & CIFAR-10 & MNIST & SVHN & CIFAR-10 & MNIST & SVHN & CIFAR-10 \\
\midrule
\texttt{None} & 0.01\pp0.00 & 0.05\pp0.00 & 0.17\pp0.00 & 0.02\pp0.00 & 0.16\pp0.00 & 0.50\pp0.00 & 0.00\pp0.00 & 0.04\pp0.00 & 0.15\pp0.00 \\
\midrule
\texttt{OOD} & 0.42\pp0.00 & 0.40\pp0.00 & 0.36\pp0.00 & 1.24\pp0.00 & 1.16\pp0.00 & 0.98\pp0.00 & 0.53\pp0.00 & 0.42\pp0.00 & 0.66\pp0.00  \\
\texttt{PINV} & \textbf{0.67\pp0.01} & \textbf{0.72\pp0.00} & \textbf{0.56\pp0.00} &\textbf{ 1.91\pp0.04} & \textbf{1.96\pp0.00} & \textbf{1.56\pp0.00} & \textbf{0.80\pp0.00} & \textbf{0.79\pp0.00} & \textbf{0.72\pp0.00}  \\
\texttt{EMD} & 0.40\pp0.00 & 0.60\pp0.00 & 0.43\pp0.00 & 1.21\pp0.01 & 1.69\pp0.02 & 1.21\pp0.01 & 0.40\pp0.00 & 0.55\pp0.00 & 0.50\pp0.01 \\
\texttt{DF} & 0.00\pp0.00 & 0.01\pp0.00 & 0.00\pp0.00 & 0.00\pp0.00 & 0.01\pp0.00 & 0.01\pp0.00 & 0.00\pp0.00 & 0.00\pp0.00 & 0.01\pp0.00 \\
\bottomrule
\end{tabular}%
}
\caption{\textbf{Expected attack advantage on ResNet-18 architecture with attack set of size 100.} \texttt{PINV} outperforms almost all other attacks by a significant margin across dataset and learning event.}
\label{fig:resnet-100-events}
\end{table}

\subsection{Random Noise}

Table~\ref{fig:random-table-curv} shows \texttt{PINV} compared with random noise in curvature scoring over ResNet-18 architecture. Despite lack of semantic meaning, \texttt{PINV} supports our theory linking high sensitivity queries with high memorization.

\begin{table}[H]
\centering
\begin{tabular}{c|ccc}
\toprule
\textbf{Attack} & Attack set size 10 & Attack set size 100 & Attack set size 1000 \\
\midrule
\texttt{Random Noise} & 0.13\pp0.00 & 0.05\pp0.00 & 0.01\pp0.00 \\
\texttt{PINV} & 0.36\pp0.00 & 0.21\pp0.00 & 0.13\pp0.00 \\
\bottomrule
\end{tabular}
\caption{\textbf{Curvature score of ResNet-18 architecture on CIFAR10.} \texttt{PINV} outperforms simple random noise in memorization scoring.}
\label{fig:random-table-curv}
\end{table}

\subsection{Boundary Points in DeepFool}

One might believe that samples that are most likely to change the decision boundary would be more memorized. We consider attack sets for \texttt{DF} consisting of what we call \textit{boundary points}. We define a sample as being closer to the boundary when the model's maximum softmax probability is smaller, indicating lower confidence for a model in its prediction on that sample. We find that this selection of attack set improves \texttt{DF}, but not enough to gain advantage over our simpler \texttt{PINV} that retains effectiveness without knowledge of the underlying dataset. Figure~\ref{fig:cifar10_vgg_boundary} displays the performance of \texttt{DF} on both random and boundary starting points, both of which fall short of our best performer in \texttt{PINV}.

\begin{figure}[h!]
    \centering
    \includegraphics[width=0.6\linewidth]{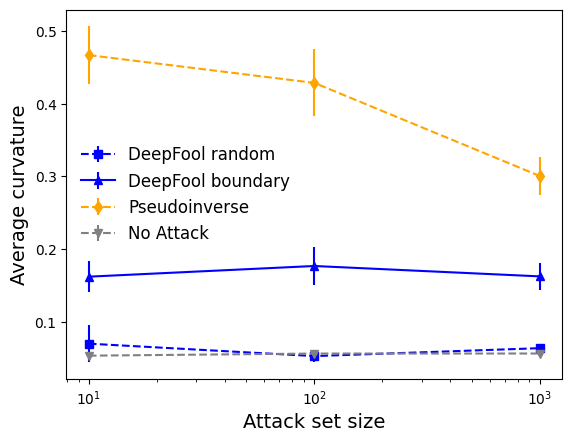}
    \caption{\textbf{Performance of \texttt{DF} on random vs. boundary starting points.} Although \texttt{DF} improves performance with smart selection of starting points, \texttt{PINV} remains superior without any knowledge of the underlying distribution.}
    \label{fig:cifar10_vgg_boundary}
\end{figure}

\subsection{Percentile}
\label{sec:appendix-percentile}

\begin{table}[H]
\centering
\resizebox{\textwidth}{!}{%
\begin{tabular}{c|ccc|ccc|ccc}
\toprule
\textbf{Attack} & \multicolumn{3}{c|}{\textbf{Loss Curvature}} & \multicolumn{3}{c|}{\textbf{Confidence Event}} & \multicolumn{3}{c}{\textbf{Privacy Score}} \\
\cmidrule(r){2-4} \cmidrule(r){5-7} \cmidrule(r){8-10}
& MNIST & SVHN & CIFAR-10 & MNIST & SVHN & CIFAR-10 & MNIST & SVHN & CIFAR-10 \\
\midrule
\texttt{None} & 50.0\pp0.9 & 49.2\pp7.9 & 51.2\pp3.5 & 51.3\pp5.9 & 50.6\pp19.8 & 49.3\pp12.0 & 50.4\pp5.0 & 49.2\pp3.9 & 49.6\pp15.2 \\
\midrule
\texttt{OOD} & 98.9\pp0.1 & 84.0\pp4.2 & \textbf{84.1\pp1.8} & 99.8\pp0.0 & 97.4\pp0.0 & 97.5\pp0.3 & 61.0\pp23. & 24.5\pp3.9 & 68.0\pp23.5  \\
\texttt{PINV} & \textbf{99.1\pp0.1} & \textbf{97.9\pp0.7} & 78.2\pp8.6 & \textbf{99.9\pp0.0} & \textbf{99.6\pp0.0} & \textbf{99.3\pp0.0} & \textbf{79.4\pp24.0} & \textbf{100.0\pp0.0} & \textbf{98.9\pp0.7} \\
\texttt{EMD} & 97.7\pp1.0 & 63.3\pp3.5 & 48.7\pp2.4 & 99.8\pp0.0 & 99.0\pp0.0 & 96.7\pp1.1 & 47.9\pp23.7 & 42.1\pp8.3 & 63.8\pp11.6 \\
\texttt{DF} & 50.6\pp8.6 & 53.1\pp2.9 & 50.7\pp14.1 & 50.1\pp2.6 & 54.3\pp8.3 & 50.4\pp3.8 & 55.6\pp8.4 & 45.7\pp0.5 & 60.6\pp12.7 \\
\bottomrule
\end{tabular}%
}
\caption{\textbf{Average attack score percentile on ResNet-18 architecture with attack set of size 100.} \texttt{PINV} outperforms almost all other attacks across dataset and proxy.}
\label{fig:resnet-100-precentile}
\end{table}

Table~\ref{fig:resnet-100-precentile} displays the mean and variance of score percentile before and after our four attacks for the ResNet-18 architecture with attack size of 100. Unsurprisingly, \texttt{PINV} remains the best attack across training setting, in line with our other results. However, when looking at only percentile, the gap between \texttt{PINV} and other attacks is not as dramatic as when looking at raw scores. We find that the percentile results affirms~\citet{feldman2020does}'s observations that few samples have exhibit strong memorization in the form of a long-tail distribution. While the magnitude of curvature of our attack points is small, our attacks display effectiveness in creating highly ranked samples. \texttt{PINV} remains the most desirable attack with its computational efficiency and weaker threat model.

\subsection{Test Correctness}

Table~\ref{fig:accuracy-cifar10} displays the mean and variance of test accuracy of a model before and after each attack over $t=5$ trials over CIFAR10. Unsurprisingly, the test accuracy does not change significantly with our attacks, considering we only perturb a small fraction (2\%) of the training data set at most. Our attacks would remain undetected by any test accuracy analysis.

\begin{table}[H]
\centering
\resizebox{\textwidth}{!}{%
\begin{tabular}{c|ccc|ccc|ccc}
\toprule
\textbf{Model} & \multicolumn{3}{c|}{\textbf{VGG-11}} & \multicolumn{3}{c|}{\textbf{ResNet-18}} & \multicolumn{3}{c}{\textbf{MobileNetV2}} \\
\cmidrule(r){2-4} \cmidrule(r){5-7} \cmidrule(r){8-10}
Attack set size & 10 & 100 & 1000 & 10 & 100 & 1000 & 10 & 100 & 1000 \\
\midrule
\texttt{None} & 76.4\pp0.9 & 77.1\pp0.2 & 76.8\pp0.6 & 78.9\pp0.2 & 79.6\pp0.0 & 78.8\pp0.6 & 69.4\pp0.6 & 69.7\pp0.5 & 68.8\pp0.4 \\
\midrule
\texttt{OOD} & 75.6\pp6.2 & 76.3\pp1.5 & 75.7\pp0.3 & \textbf{79.2\pp0.0} & 78.4\pp2.5 & 77.1\pp0.4 & 68.8\pp1.5 & 69.6\pp1.1 & 69.0\pp0.3  \\
\texttt{PINV} & \textbf{78.0\pp1.0} & \textbf{76.6\pp0.1} & 76.4\pp1.2 & 79.1\pp0.4 & 78.3\pp0.3 & 77.9\pp0.5 & 67.8\pp0.2 & 68.7\pp1.5 & 68.8\pp0.3  \\
\texttt{EMD} & 76.3\pp2.5 & 76.0\pp1.3 & \textbf{76.5\pp0.43} & 79.0\pp0.5 & 79.2\pp0.5 & 76.5\pp2.2 & 68.7\pp2.0 & 69.4\pp1.3 & 69.1\pp1.1 \\
\texttt{DF} & 77.1\pp1.7 & 77.2\pp1.3 & 76.6\pp1.0 & 78.9\pp0.6 & \textbf{79.3\pp0.1} & \textbf{79.5\pp0.5} & \textbf{69.1\pp0.4} & \textbf{70.4\pp0.4} & \textbf{70.1\pp1.8} \\
\bottomrule
\end{tabular}%
}
\caption{\textbf{Test accuracy before and after attack on CIFAR10.} All attacks have negligible influence on testing accuracy.}
\label{fig:accuracy-cifar10}
\end{table}

\subsection{Vision Transformers}
\label{sec:appendix-vit}

\begin{table}[H]
\centering
\begin{tabular}{c|ccc}
\toprule
\textbf{Attack} & Attack set size 10 & Attack set size 100 & Attack set size 1000 \\
\midrule
\texttt{None} & 0.01\pp0.00 & 0.01\pp0.00 & 0.01\pp0.00 \\
\texttt{PINV} & 0.03\pp0.00 & 0.03\pp0.00 & 0.04\pp0.00 \\
\bottomrule
\end{tabular}
\caption{\textbf{Curvature score of ViT architecture on ImagetNet.} \texttt{PINV} retains its high attack scoring capability over non-convolutional networks and higher resolution datasets.}
\label{fig:vit-table-curv}
\end{table}

\begin{table}[H]
\centering
\begin{tabular}{c|ccc}
\toprule
\textbf{Attack} & Attack set size 10 & Attack set size 100 & Attack set size 1000 \\
\midrule
\texttt{None} & 0.02\pp0.00 & 0.02\pp0.00 & 0.02\pp0.00 \\
\texttt{PINV} & 0.87\pp0.00 & 0.87\pp0.00 & 0.86\pp0.00 \\
\bottomrule
\end{tabular}
\caption{\textbf{Confidence event score of ViT architecture on ImagetNet.} \texttt{PINV} retains its high attack scoring capability over non-convolutional networks and higher resolution datasets.}
\label{fig:vit-table-conf}
\end{table}

Tables \ref{fig:vit-table-curv} and \ref{fig:vit-table-conf} display the curvature and confidence scoring of ImageNet~\citep{russakovsky2015imagenet} data before and after the \texttt{PINV} attack on the Vision Transformer (ViT) architecture~\citep{dosovitskiy2020image}. Our results validate our theoretical claims that high sensitivity queries are model and dataset agnostic in attacking memorization scores.

\subsection{Text Classification}
\label{sec:appendix-bert}

 \begin{table}[H]
\centering
\begin{tabular}{c|ccc}
\toprule
\textbf{Attack} & Attack set size 10 & Attack set size 100 & Attack set size 1000 \\
\midrule
\texttt{None} & 0.04\pp0.00 & 0.04\pp0.00 & 0.04\pp0.00 \\
\texttt{INV} & 0.55\pp0.00 & 0.61\pp0.00 & 0.39\pp0.00 \\
\bottomrule
\end{tabular}
\caption{\textbf{Confidence event score of BERT architecture on AG News classification dataset.} Inverse operations (\texttt{INV}) retain its high attack scoring capability over textual data.}
\label{fig:agnews-table-conf}
\end{table}

Table~\ref{fig:agnews-table-conf} displays confidence scoring on AG News~\citep{zhang2015character} text classification dataset over BERT~\citep{devlin2019bert} transformer architecture. The \texttt{INV} attack describes the additive inverse operation on the additive integer group modulo size vocab length described by a pretrained tokenizer. Our results show that inverse operations display high memorization scoring agnostic of data modality.

\end{document}